\documentclass{article}

\usepackage{microtype}
\usepackage{graphicx}
\usepackage{subfigure}
\usepackage{booktabs} 

\usepackage{mathtools}
\usepackage{dsfont}

\usepackage{amsthm} 
\usepackage{amsmath}

\usepackage{algorithm}
\usepackage{algpseudocode}
\usepackage{algpascal}
\usepackage{adjustbox}
\usepackage{float}
\usepackage{dblfloatfix}

\usepackage{hyperref}

\usepackage[accepted]{icml2023}

\usepackage{amsmath}
\usepackage{amssymb}
\usepackage{mathtools}
\usepackage{amsthm}

\usepackage[capitalize,noabbrev]{cleveref}

\theoremstyle{plain}
\newtheorem{theorem}{Theorem}[section]

\newtheorem{lemma}[theorem]{Lemma}

\theoremstyle{definition}

\theoremstyle{remark}

\usepackage[textsize=tiny]{todonotes}

\icmltitlerunning{Seeking Next Layer Neurons’ Attention for Error-Backpropagation-Like Training in a Multi-Agent Network Framework}

\newcommand\normf[1]{{\left\Vert#1\right\Vert}_F}
\newcommand\normx[1]{{\left\Vert#1\right\Vert}_2}

\newcommand\fdot[2]{\langle #1,#2 \rangle_F}
\newcommand\conv[2]{\text{Conv}(#1,#2)}
\newcommand\deconv[2]{\text{ConvT}(#1,#2)}

\newcommand\rconv[2]{\psi(#1,#2)}

\newcommand\normalizeT[1]{\text{Normalized}(#1)}
\newcommand\diag[1]{\text{diag}(#1)}

\DeclareMathOperator*{\argmax}{arg\,max}

\def\algname{BackMAN}

\begin{document}

\twocolumn[
\icmltitle{Seeking Next Layer Neurons’ Attention for Error-Backpropagation-Like Training in a Multi-Agent Network Framework}

\begin{icmlauthorlist}
\icmlauthor{Arshia Soltani Moakhar}{sharif}
\icmlauthor{Mohammad Azizmalayeri}{sharif}\\
\icmlauthor{Hossein Mirzaei}{sharif}
\icmlauthor{Mohammad Taghi Manzuri}{sharif}
\icmlauthor{Mohammad Hossein Rohban}{sharif}
\end{icmlauthorlist}

\icmlaffiliation{sharif}{Department of Computer Engineering, Sharif University of Technology, Tehran, Iran}

\icmlcorrespondingauthor{Arshia Soltani Moakhar}{arshia.soltani@sharif.edu}
\icmlcorrespondingauthor{Mohammad Azizmalayeri}{m.azizmalayeri@sharif.edu}

\icmlkeywords{Machine Learning, ICML, Multi-agent, credit assignment}

\vskip 0.3in
]

\printAffiliationsAndNotice{}  

\begin{abstract}
Despite considerable theoretical progress in the training of neural networks viewed as a  multi-agent system of neurons, particularly concerning biological plausibility and decentralized training, their applicability to real-world problems remains limited due to scalability issues. In contrast, error-backpropagation has demonstrated its effectiveness for training deep networks in practice. In this study, we propose a \textbf{local objective} for neurons that, when pursued by neurons individually, align them to exhibit similarities to error-backpropagation in terms of efficiency and scalability during training. 
For this purpose, we examine a neural network comprising decentralized, self-interested neurons seeking to maximize their local objective ---attention from subsequent layer neurons--- and identify the optimal strategy for neurons. We also analyze the relationship between this strategy and backpropagation, establishing conditions under which the derived strategy is equivalent to error-backpropagation. Lastly, we demonstrate the learning capacity of these multi-agent neural networks through experiments on three datasets and showcase their superior performance relative to error-backpropagation in a catastrophic forgetting benchmark.
\end{abstract}

\section{Introduction}
\label{submission}
Deep learning has achieved state-of-the-art performance across various domains. For example, in computer vision, deep learning models have outperformed human capabilities in image recognition and generation tasks \cite{dodge2017study, vit, resnet, imageGeneration}. In natural language processing, deep learning models can generate text \cite{llama, radford2018improving}, which is often indistinguishable from human-authored content.
Additionally, they have defeated world champions in games with vast search spaces like Go and Dota 2 \cite{go, dota2}.

Despite these achievements, deep learning models are not without their limitations. As we continue to develop and refine these models, we encounter fundamental drawbacks that contrast with the adaptability and resilience observed in natural systems. For example, deep learning models can be easily deceived in image recognition tasks by imperceptible alterations to input images, such as modifying a few pixels, that would not mislead human observers \cite{szegedy2014intriguing, goodfellow2015explaining}. Furthermore, these models struggle to adapt to slightly different situations without undergoing retraining. An instance of this limitation is the significant decline in the performance of speech recognition models in the presence of background noise \cite{chen2022noiserobust}.

In contrast, intelligence exhibited by collective systems, as observed in nature, tends to be adaptive, robust, and makes fewer rigid assumptions about environmental configurations \cite{collectiveIntelligenceSurvay}. For example, such behaviors and properties can be seen in ant colonies, which can readily adapt to novel environments. Inspired by these characteristics, there has been a growing interest in developing neural networks that emulate these properties through the employment of multi-agent systems in the machine-learning domain \cite{lowe2020multiagent, jaques2019social, Chalk598086}.

In the context of multi-agent systems within machine learning, two primary perspectives can be identified. The first perspective involves constructing an environment wherein multiple agents interact with one another, with each agent employing a neural network to make decisions \cite{MARLSurvey,suarez1, suarez2}. Although some research within this perspective shares a single neural network across multiple agents, the neural network is consistently regarded as a unified entity \cite{MARLsinglenetwork1, MARLsinglenetwork2}. The second perspective, which is the focus of this paper, treats a single neural network as a multi-agent system itself. In this view, individual neurons \cite{AuctionNeuron}, layers \cite{mostafa2017deep}, or segments of layers within a neural network are regarded as discrete agents that communicate with each other by exchanging messages throughout the network, ultimately resulting in the neural network's output.

Three research fields have explored modeling neural networks as collections of individual agents: biologically plausible networks, collective intelligence systems, and decentralized neural networks. Inspired by the learning in biological systems, some recent work sought an alternative to the error-backpropagation algorithm, given ample evidence that error-backpropagation does not occur in the brain \cite{naturereviewsBP,Richards2019,benjio}. Consequently, several papers suggest utilizing multi-agent modeling for biologically plausible learning approaches \cite{ott2020giving, balduzzi2014kickback,balduzzi2014cortical,Lewis2014}. From a collective intelligence standpoint, this modeling is intriguing, as numerous simple neurons (e.g., linear functions) communicate and achieve higher intelligence capable of, for instance, classifying images \cite{collectiveIntelligenceSurvay}. Lastly, another study employs reinforcement learning agents that auction their output to the next layer neurons, aiming to create a decentralized neural network with components managed by different individuals \cite{AuctionNeuron}.

Although these multi-agent modeling studies present intriguing motivations, they fail to fully realize their potential due to their limited ability to scale in deep networks. Deep neural networks are essential components of most machine learning advances \cite{szegedy2014going} and are critical properties of animal brains. Therefore without scalability, the practical application or usage of these multi-agent studies as a basis for understanding animals brains or deep neural networks is minimal.

In contrast, error-backpropagation can readily train deep neural networks \cite{szegedy2014going} and achieve state-of-the-art performance across numerous tasks \cite{vit,llama,speechRecognition}. This characteristic of error-backpropagation motivates us to develop a method that adopts a multi-agent perspective on neural networks while leveraging an error-backpropagation-like algorithm to train deeper networks.

On this basis,  we propose a novel and straightforward local objective for self-interested agents in a multi-agent neural network that closely resembles the behavior of error-backpropagation and gradient descent in the model.  In our method, \textbf{each neuron is considered an agent and aims to maximize the attention it receives from subsequent layer neurons}. Attention is measured by the $\ell_2$-norm of the weight vector connecting a neuron to the next layer of neurons. We refer to the objective as ``local'' because it is only dependent on the environment that surrounds each neuron. We demonstrate in Theorem \ref{semi_is_nash_label} that under certain reasonable assumptions, the proposed objective for the agents (neurons) results in a training algorithm for the network which is similar to error-backpropagation. We call this new training algorithm \algname{}.

In \algname{}, neurons combine the backpropagated signals (gradients) from subsequent layers in a manner identical to error-backpropagation. However, unlike error-backpropagation, each neuron multiplies the signals by a scalar before passing them to the preceding layer. This scaling allows neurons to consider the preferences of neurons in the subsequent layer equally, in contrast to error-backpropagation, where some neurons may propagate stronger gradients.

Although the backpropagated signals are calculated similar to gradients, we omit the use of term ``gradients'' and use ``signals'' instead. To check the similarities between our proposed method and error-backpropagation, we investigate conditions under which the scalar is one for all neurons, implying conditions where the resulting behavior equates to error-backpropagation. Although this condition may be unrealistic, we contend that this equivalence could help in understanding error-backpropagation from a multi-agent perspective.

Furthermore, we apply our theoretical understanding to practice, using MNIST, CIFAR-10, and CIFAR-100~\cite{mnist, cifar} datasets to achieve results comparable to error-backpropagation with identical model architectures. Additionally, to demonstrate that our proposed learning algorithm exhibits behaviors that could be expected from multi-agent neural networks, we evaluate our method on split-MNIST, a catastrophic forgetting benchmark. We observe improvement over plain backpropagation, resembling the performance of other biologically inspired learning algorithms \cite{zenke2017continual, synapticToforgetting}.

In summary, our contributions have four parts. 
First, our model helps in understanding error-backpropagation by relating it to multi-agent systems, supported by our theoretical findings. 
Second, our proposed learning algorithm exhibits improvement over error-backpropagation in a forgetting task, which aligns with the interests of multi-agent systems and biologically plausible neural networks. This suggests that the proposed objective for neurons merits further investigation in multi-agent systems. 
Third, our algorithm is the first among multi-agent neural network algorithms capable of training relatively deeper neural networks. This study paves the way to advance the field of multi-agent machine learning through further exploration of the synergistic relationship between multi-agent neural networks and error-backpropagation.
Lastly, and most significantly, we introduce an objective that is \textbf{local} and solely reliant on the surrounding environment of neurons. This objective encourages Self-Interested neurons to enhance the collective performance of the neural network.

\section{Related Work}
\textbf{Error-backpropagation as a multi-agent system trainer}: The authors of \cite{balduzzi2014kickback}, akin to our work (Section \ref{Relation to gradient descent:label}),
proposed a scenario in which self-interested neurons behave as though the system were trained using error-backpropagation. Nonetheless, in their proposed scenario, gradients of the neurons' output with respect to the global loss must be externally supplied to neurons at each iteration, as these signals influence the neurons' objective. Therefore, to align their proposed algorithm with gradient descent, this number should be calculated for each neuron on each sample and then supplied to the neuron. These information-rich signals contrast with our work, in which neurons receive less informative signals in cases where the system trains, like error-backpropagation. In our setting, each neuron will be supplied with only one number regardless of the number of samples in that iteration to act as error-backpropagation. 

\textbf{$\ell_2$-norm as worth}: In \cite{Lewis2014}, the authors propose a biologically inspired multi-agent system. In their system, neurons try to maximize their worth, which is determined by the difference in model loss if a specific neuron is removed from the network. They demonstrate that neurons can approximate their worth by $\ell_2$-norm of their output connections, akin to our work. However, they did not study conditions under which all neurons choose their optimal actions since their primary goal was to find learning rates - plasticity ratio - for synaptic connections rather than proposing a training algorithm.

\textbf{Multi-Agent Reinforcement Learning}: In the field of multi-agent reinforcement learning, researchers typically focus on how agents \textbf{learn} the strategies \cite{LearningToShare, MARL2,  MARL3, MARL4, MARL5, MARL6}. This contrasts with our work, in which we \textbf{theoretically identify} the best strategy for agents. We hypothesized that these theoretically identified strategies could be learned through evolutionary processes or in an alternative way. Learning these strategies could be regarded as meta-learning, which has been heavily investigated recently \cite{meta1, meta2, meta3}. However, this discussion is beyond the scope of this paper.

\textbf{Reinforcement Learning Applied to Neurons}: Some studies explore cases where reinforcement learning is applied to individual neurons. In \cite{AuctionNeuron}, neurons conduct auctions to buy/sell signals from/to other neurons. They justify this approach by proposing its application in the distributed training of neural networks. Conversely, in \cite{wang2015reinforcement}, neurons perform Gibbs sampling and attempt to minimize model loss. The author of \cite{ott2020giving} proposes a cooperative-competitive framework where each neuron is an individual agent. On the cooperative side, these neurons attempt to optimize to enhance the entire network’s performance. On the competitive side, neurons receive a local loss based on factors such as sparsity. The underlying intuition is that a layer's neuron activations should be sparse since a layer providing uniformly active neuron activations does not convey any information to subsequent layers. Consequently, neurons that violate the sparsity requirement incur a penalty.

 On a related note, \cite{balduzzi2014kickback} suggested Kickback, a biologically plausible learning algorithm in which a neuron's objective is a function of the $\ell_1$-norm of its output weights and model loss. However, they could only train networks with three hidden layers and in limited settings with only two outputs. The primary distinction between these works and our study is that neurons must learn their best response in these works. In contrast, we theoretically determine their optimal strategy and have neurons employ this strategy from the onset of training. This characteristic enables us to train larger networks compared to the abovementioned works.

\textbf{Layerwise Learning}: Unlike previous works, layerwise learning approaches focus on training each layer independently, proposing distinct objectives for each layer rather than each neuron. \cite{mostafa2017deep} introduced a framework in which the output of each layer is connected to a linear layer with random weights. The objective of each layer is to minimize the classification loss of the linear classifier connected to its output. Therefore, the objective of each layer differs from those of other layers. \cite{nøkland2019training} adopted a similar framework, with each layer being trained using a combination of classification loss and similarity loss between mini-batch samples. \cite{hinton2022forwardforward} also implemented a similar setup; however, each layer classifies between in- and out-distribution data using the norm of the layer's output.


\section{Method}

In this section, we describe our approach to modeling a neural network as a multi-agent system. In Section~\ref{notation}, we introduce our notation and establish the context for our analysis. Critically, within Section~\ref{notation}, we define the utility function of neurons (the function they seek to maximize)  and the network weights, which each neuron controls, along with their weight update limitations (referred to as the action space in game theory literature). 

In Section~\ref{Semi-gradient-descent algorithm:label}, we introduce the \algname{} (\textbf{Back}propagation aware \textbf{M}ulti-\textbf{A}gent \textbf{N}etwork) training algorithm. \algname{} directly stems from our framework, which describes the behavior of neurons and how the network weights are updated, assuming each neuron acts optimally (neuron rationality) within the proposed framework.

Section~\ref{Relation to gradient descent:label} seeks to demonstrate that \algname{} and error-backpropagation are notably similar. Thus, this section investigates a condition under which the system behaves as if it were trained by gradient descent. The similarity between \algname{} and error-backpropagation supports our framework, because even with a group of self-interested, decentralized neurons, their behavior is same as error-backpropagation algorithm, which is centralized. Furthermore, this similarity suggests that error-backpropagation can be understood and analyzed through a multi-agent framework.

Finally, Section~\ref{Nash Equilibrium:label} presents our primary theoretical result, elucidating the rationale behind why \algname{} emerges as the neuronal behavior in the proposed framework. More specifically, Section~\ref{Nash Equilibrium:label} demonstrates that \algname{} is a Nash equilibrium for the proposed setting under some reasonable conditions.  Therefore, under these conditions, each neuron's best response is to act according to \algname{}, and since they are rational and act optimally, they will choose the \algname{} response.

\subsection{Setting}
\label{notation}
\textbf{Notation}: We consider a multilayer perceptron (MLP) network  and a dataset with $m$ samples, each having $n_x$ features. The input matrix is represented by $X \in \mathds{R}^{m \times n_x}$. In this context, the weight $w^{[l]}_{i,j}(\tau)$ signifies the connection weight between neuron $j$ in layer $l-1$ and neuron $i$ in the subsequent layer at iteration $\tau$. The output value of layer $l$ is denoted by $a^{[l]}(\tau) \in \mathds{R}^{m \times n_l}$ and pre-activation values denoted by $z^{[l]}(\tau) \in \mathds{R}^{m \times n_l}$. Therefore, $a^{[l]}_{k,i}(\tau)$ represents the output of neuron $i$ of layer $l$ for sample $k$ at iteration $\tau$. Assuming that neurons in layer $l$ have activation function $f^{[l]}$, we can write 
\begin{equation}
\label{forward_pass}
a^{[l]}(\tau)=f^{[l]}(a^{[l-1]}(\tau). {w^{[l]}}^\top(\tau) ) = f^{[l]}(z^{[l]}(\tau)).
\end{equation}
For simplicity, 
we may exclude $\tau$ from the equations; however, when omitted, 
it implicitly refers to the iteration $\tau$. $M^{[l]} \in\mathds{R}^{m \times n_l}$ is a matrix indicating the sloop of activation function in samples for neurons of layer $l$. More formally 
$M^{[l]}_{k,i} = {f^{[l]}}^{'} (z^{[l]}_{k,i})$.
Thus, in the case where neurons use identity function, 
$\forall_{k,i} M^{[l]}_{k,i} = 1 $.
Also, when a neuron employs ReLU as its activation function,  
$M^{[l]}_{k,i} = \mathds{I}(a^{[l]}_{k,i} > 0)$.

Additionally, let $\ell$ denote the network loss, and  $L$~signifies 
the total number of layers in the model. Furthermore, if $A \in \mathds{R}^{a \times b}$ be an arbitrary matrix, 
then $A_i$ denotes row $i$ of matrix $A$, such that $A_i \in \mathds{R}^{1 \times b}$. Similarly, $A_{:i}$ represents column $i$, such that $A_{:i} \in \mathds{R}^{a \times 1}$.

\textbf{Setup}: In our setting, each neuron controls the weights of its input connections, so $\forall_j w^{[l]}_{i,j}$ is controlled by neuron $i$ in layer $l$. Neurons modify their input connection weights simultaneously after each forward pass. The utility function of these neurons is defined as the $\ell_2$ norm of its output connection.
Specifically, the objective function of neuron $i$ in layer $l$ is defined by 
\begin{equation}
\label{utility_equation:body}
    U^{[l]}_{i}(\tau)=\normx{w^{[l+1]}_{:i}}=\sqrt{\sum_k (w^{[l+1]}_{k,i}(\tau))^2}.
\end{equation}
 The last layer neurons do not have output connections, so this definition is not applicable. We assume those neurons' objectives are to minimize the model loss as
\begin{equation}
    U^{[L]}_{i}(\tau)=-\ell(\tau).
\end{equation}
Intuitively, a neuron utility is maximized only if the neurons in the subsequent layer find that neuron more informative and increase the weight of the connection to that neuron. This happens only if that neuron can update its input weights such that more useful information is passed through its output. 

As agents control only their input weights, an \textit{agent's response} refers to the weight updates it performs on its weights. So the response of neuron $i$ in layer $l$ at iteration $\tau$ is equal to $w^{[l]}_i(\tau+1)-w^{[l]}_i(\tau)$. The \textit{best response} is the action that produces the most favorable outcome for an agent, considering other agents' responses as given \cite{bestresponse}.

\textbf{Assumptions}: As commonly found in game theory papers, we assume that all neurons -agents- are rational, meaning that neurons choose the strategy that maximizes their expected utility. However, in our setting, neurons are also greedy, which implies that they value short-term results and may not necessarily consider the long-term consequences of their choices. In the theoretical analysis, we assume that the forward passes are applied to the entire dataset rather than to batches or single samples. Additionally, each neuron has a constraint on the $\ell_2$ norm of the updates it applies to its input weights. Specifically, if neuron $i$ updates its input weight from $w^{[l]}_i(\tau)$ to $w^l_i(\tau+1)$, we require 
\begin{equation}
    \normx{w^{[l]}_i(\tau)-w^{[l]}_i(\tau+1)} \leq \alpha^{[l]}_i(\tau),
    \label{updateAssumption}
\end{equation}
where $\alpha^{[l]}(\tau) \in \mathds{R}^{n_l}$ is a predefined constant.

\subsection{\algname{} Algorithm}
\label{Semi-gradient-descent algorithm:label}
In the previous Section, the framework within which neurons operate was delineated. This Section explains how neurons would behave if they operate rationally within this framework. We term this behavior \algname{} and delve into it in this section.

Neurons adjust their connections weights to neurons in the preceding layer based on the values generated by these preceding neurons for each samples. Since the neurons utility is based upon their output connection weights, their utility is, by extension, indirectly dependent upon the values they generate per sample. Consequently, neurons have preferences regarding how their generated values should change to maximize their future utility. \algname{} calculates these neuronal preferences, which we denote as 
$g$, and updates network weights based on these preferences.

More formally, the matrix $g^{[l]} \in \mathds{R}^{m \times n_l}$ represents the neurons' preferences regarding their pre-activation values’ direction, aiming to maximize their utility. For example, if $g^{[l]}_{s,i} < 0$, neuron $i$ in layer $l$ wants its pre-activation value on sample $s$ to decrease. Conversely $g^{[l]}_{s,i} > 0$ indicates a desire for an increase.

Using $g$ values, the weights are updated as follows:
\begin{equation}\label{eq:update_rule}
    w^{[l]}(\tau+1) = w^{[l]} + {g^{[l]}}^\top.a^{[l-1]}.
\end{equation}  ${g^{[l]}}^\top.a^{[l-1]}$ allows neurons to update their input weights in a direction that aligns with their pre-activation direction preferences.

${g^{[l]}}^\top.a^{[l-1]}$ updates the weights and we had an assumption regarding limitation on the amount of update each neuron can apply to its weights. Therefore, a normalization factor should be applied to neurons' preferences before using them to update weights. Thus, $g$ is calculated as:
\begin{equation}
\label{genral_g}
    g^{[l]} \coloneqq  r^{[l]}. diag(c^{[l]}),
\end{equation}
where $r^{[l]}$ can be interpreted as an unnormalized counterpart of $g^{[l]}$, and $c^{[l]} \in \mathds{R}^{n_l}$ as the normalization factor that enforces 
the update assumption (Eq.~\ref{updateAssumption}) in the update rule (Eq.~\ref{eq:update_rule}). 

The normalization factor $c^{[l]}$ can be calculated by re-writing Assumption~\ref{updateAssumption} as:
\begin{equation*}
   \begin{aligned}
       \alpha^{[l]}_i\geq\normx{w^{[l]}_i(\tau)-w^{[l]}_i(\tau+1)}=\normx{{{(g^{[l]}}^\top.a^{[l-1]})}_i} = \\ \normx{{{(( r^{[l]}. diag(c^{[l]}))}^\top.a^{[l-1]})}_i}= c^{[l]}_i\normx{{{{(r^{[l]}}^\top.a^{[l-1]})}_i}} \\\Rightarrow c^{[l]}_i \leq \frac{\alpha^{[l]}_i}{\normx{{{(r^{[l]}}^\top.a^{[l-1]})}_i}},
\end{aligned} 
\end{equation*}
with $\alpha^{[l]}_i$ being the predefined constant in \eqref{updateAssumption}. So it suffices that we select $c^{[l]}_i$ as:
\begin{equation}
\label{general_c}
    c^{[l]}_i = \frac{\alpha^{[l]}_i}{\normx{{r^{[l]}_{:i}}^\top.a^{[l-1]}}},
\end{equation}
to enforce Assumption~\ref{updateAssumption} in our updates.

Regarding the calculation of unnormalized neurons’ preferences, or $r$, for the last layer, 
$r^{[l]}$ is calculated as:
\begin{equation}
\label{lastLayer_R}
    r^{[L]}  \coloneqq  -\frac{\partial \ell}{\partial z^{[L]} },
\end{equation}
but for the other layers, $r^{[l]}$ is calculated based on $g^{[l+1]}$ as:
\begin{equation}
\label{hiddenLayer_R}
    r^{[l]}=(g^{[l+1]}.w^{[l+1]}) \odot M^{[l]}
\end{equation}

To make these equations more understandable, we try to intuitively describe their role. $r$ and $g$ are similar to gradients in error-backpropagation. Using this perspective and assuming $g^{[l]}=\frac{\partial \ell}{\partial z^{[l]}} $ Eq.~\ref{eq:update_rule} aligns with the gradient descent update since:
\[ 
{g^{[l]}}^\top.a^{[l-1]} = \frac{\partial \ell}{\partial z^{[l]}}.a^{[l-1]} = \frac{\partial \ell}{\partial w^{[l]}}.
\]
Eq.~\ref{lastLayer_R} clearly indicates the unnormalized neurons’ preferences for pre-activation values, and requires no further clarification.

Further, if we assume $g^{[l+1]}=\frac{\partial \ell}{\partial z^{[l+1]}}$, Equation~\ref{hiddenLayer_R} equals the chain rule since:
\begin{equation*}
   \begin{aligned}
r^{[l]}=(g^{[l+1]}.w^{[l+1]}) \odot M^{[l]}= (\frac{\partial \ell}{\partial z^{[l+1]}}.w^{[l+1]}) \odot M^{[l]}\\ = 
\frac{\partial \ell}{\partial a^{[l]}} \odot  {f^{[l]}}^{'} (z^{[l]})= \frac{\partial \ell}{\partial z^{[l]}}.
   \end{aligned}
\end{equation*}

Although we assumed $g^{[l+1]}=\frac{\partial \ell}{\partial z^{[l+1]}}$ to intuitively explain the role of the equations, this assumption is not far from being accurate.  In Lemma~\ref{g_equal_delta_lemma} in Appendix \ref{relation to gradient descent:Appendix}, we have stated conditions where $g$ is equal to gradients. $r$ is the unnormalized counterpart of $g$, meaning if column $i$ of $r$ is multiplied by $c_i$, it will result in column $i$ in $g$. Also, $c_i$ comes to enforce \eqref{updateAssumption}.

Collectively, these equations delineate how the weights of input connections of each neuron in the network are updated, considering the constraints imposed on neurons' behavior, such as inequality~\eqref{updateAssumption}. In Algorithm \ref{alg:semiGD}, we summarize the \algname{} algorithm for the case that neurons use a  piecewise linear activation as activation. For CNNs we provide the algorithms in Appendix~\ref{Algorithms:Appendix}. 

\begin{algorithm}
\caption{One Iteration in \algname{} Algorithm}  \label{alg:semiGD}
\begin{algorithmic}[0] 
\State $a^{[0]} \gets X$ \Comment{forward pass}
\For{l:=1}{L} 
    \State $a^{[l]} \gets f^{[l]}(a^{[l-1]}.{w^{[l]}}^\top)$
\EndFor
\\    
\State $r^{[L]} \gets -\frac{\partial \ell}{\partial z^{[L]}}$ \Comment{backward pass}
\\
\For {l:={L}}{1}
    \State $\forall_{i\leq {n_l}}: c^{[l]}_i \gets \frac{\alpha^{[l]}_i}{\normx{({r^{[l]}}^\top.a^{[l-1]})_i}}$\\ 
    \State \ \ \ \ \ \ $g^{[l]} \gets r^{[l]}.diag(c^{[l]})$ 
    \State \ \ \ \ \ \  $r^{[l-1]} \gets (g^{[l]}.w^{[l]}) \odot M^{[l]}$ 
    
\EndFor
\\
\For{l:=1}{L} \Comment{weight update}
    \State $w^{[l]} \gets w^{[l]}+ {g^{[l]}}^\top.a^{[l-1]}$
\EndFor
\end{algorithmic}
\end{algorithm}

\subsection{Relation to Gradient Descent}
\label{Relation to gradient descent:label}
As can be observed from Algorithm \ref{alg:semiGD}, the overall procedure is quite similar to gradient descent. Specifically, if \mbox{$\forall_{i,l}: c^{[l]}_i = 1$}, these two algorithms are equivalent. We prove this fact in Theorem \ref{semi_equal_GD_label}.

\begin{theorem}
\label{semi_equal_GD_label}
If \mbox{$\forall_{i,l}: c^{[l]}_i = 1$} and all neurons use a piecewise linear function as their activation function, then using algorithm \algname{} 
\mbox{$\forall_{l,i,j}: w^{[l]}_{i,j}(\tau+1)=w^{[l]}_{i,j}-\frac{\partial \ell}{\partial w^{[l]}_{i,j}}$}.
\end{theorem}

Proofs of all theorems are provided in Appendix \ref{theoreticalAppendix}. In this theorem, we prove that if $\forall_{i,l}: c^{[l]}_i=1$, our algorithm is identical to gradient descent. Since $c^{[l]}_i \propto \alpha^{[l]}_i$, by selecting the appropriate update limit -$\alpha$- for each neuron in each iteration, we could theoretically make these two algorithms identical.

\begin{table*}[t]
\caption{Performance Metrics across Multiple Datasets and Architectures. ``NA'' indicates an inability to train the network, resulting from either unavailable code from the authors or the unscalability of the approach. The table demonstrates that our method achieves comparable results to error-backpropagation.}
  \label{ClassificationTabel}
  \centering
  \resizebox{1\textwidth}{!}{
  \begin{tabular}{lcccc}
    \toprule
    Method     & MNIST & CIFAR-10 (LeNet-5) &   CIFAR-10 (ResNet-18)  & CIFAR-100 (ResNet-18) \\
    \midrule
     Error-Backpropagation (EBP) & $99.10$ & $71.24$&92.8&75.61\\
Adaptive DropConnect \cite{AuctionNeuron} &$98.64$ & $60.16$& NA & NA \\
Forward-Forward\cite{hinton2022forwardforward}&92.37& NA & NA & NA \\
Direct feedback alignment\cite{Directfeedbackalignment}&95.97&	45.79&74.37&38.94\\
Feedback alignment\cite{feedbackalignment}&	95.15&	53.31&	63.79&30.54\\
\algname{} & $98.93$ & $66.27$ & 91.33&63.00\\
    \bottomrule
  \end{tabular}
  }
\end{table*}

\begin{table*}[bh]
  \caption{This table displays the accuracy achieved for each task, and the average accuracy in the catastrophic forgetting experiment, with values, averaged over 100 seeds.}
  \label{forgetting tabel}
    \setlength{\tabcolsep}{4pt}
  \centering
  \resizebox{1\textwidth}{!}{
  \begin{tabular}{lcccccc}
    \toprule
    Method     & Task 1 & Task 2 & Task 3 & Task 4 & Task 5 & Average \\
    \midrule
     EBP&$59.41\pm18.00$&$65.63\pm18.33$&$71.15\pm16.13$& $78.57\pm12.57$&$93.39\pm6.54$&$73.47\pm4.81$\\
    \algname{}&$69.29\pm16.24$&$71.27\pm16.76$& $74.34\pm15.81$&$78.32\pm11.71$&$86.65\pm10.90$&$75.93\pm4.56$\\
    \bottomrule
  \end{tabular}
  }
\end{table*}

\subsection{Nash Equilibrium}
\label{Nash Equilibrium:label}
In this section, we demonstrate that \algname{} is a Nash equilibrium for the proposed setting under some assumptions. First, we state Theorem \ref{semi_is_nash_label}, showing that when $\alpha$ of all neurons approaches zero, the best response of a neuron, given that rest of the neurons use \algname{}, approaches the response it has in the \algname{} algorithm under some assumptions. We state these theorems for the MLP network case in Theorem \ref{semi_is_nash_label} and for Convolutional neural network (CNN) in Theorem \ref{semi_is_nash_label_CNN}. Then we show how these theorems are related to Nash Equilibrium in Theorem~\ref{convergence_theorem}. 

To this end, we define ${b^{[l]}_i}^*$ as the best response of neuron $i$ in layer $l$, given that the rest of the neurons use the \algname{} algorithm to update their weights. Additionally, $b^{[l]}_i$ refers to the response that neuron $i$ would have if it employed the \algname{} algorithm. Here are the theorems mentioned.

\begin{theorem}
\label{semi_is_nash_label}
Considering an MLP with neurons using  a piecewise linear activation, if $\forall_{l,j}: \alpha^{[l]}_j=\alpha$, and all neurons in layers after the $i$-th layer employ \algname{}, $0 < \normx{{r^{[l]}_{: i}}^\top.a^{[l-1]}}$, $g^{[l+1]}$ is independent of neuron $i$'s response at iteration $\tau-1$, and $\forall_{l,j,s}: z^{[l]}_{s,j}\notin f^{[l]} \text{breakpoints}$ then:\\
\begin{small}1)~$
\lim_{\alpha \to 0^+}\text{normalize}({{b^{[l]}_i}^*(\tau-1)})=\text{normalize}({b^{[l]}_i(\tau-1)})
$,\\
2)~$
\lim_{\alpha \to 0^+} \frac{\normx{{b^{[l]}_i}^*(\tau-1)}}{\normx{b^{[l]}_i(\tau-1)}} = 1 
$.\\
\end{small}
\normalsize
Where \textit{normalize} refers to $\ell_2$ vector normalization operator. 
\end{theorem}
Here, we discuss our assumptions and the roles they play in the proof:

The first assumption is $0 < \normx{{r^{[l]}_{:i}}^\top.a^{[l-1]}}$. This assumption is necessary since, if it does not hold, $c^{[l]}_i(\tau)$ would be undefined, as this term appears in the numerator of Equation \ref{general_c}.

The second assumption is ``$g^{[l+1]^\top}.M^{[l]}_i.a^{[l-1]}$ is independent of the response at iteration $\tau-1$''. This assumption is reasonable since $a^{[l-1]}$ is entirely independent, given that a neuron's action will only affect $a$ for subsequent layers. On the other hand, $g^{[l+1]}$ is influenced by all neurons in layer $l$, making the effect of the neuron in this term marginal. Considering $M^{[l]}$, since we are interested in small $\alpha$ values, it cannot change within a single iteration becuase of the third assumption. 

In the case of CNNs, we consider each filter in every layer as an agent. Similar to MLPs, these agents aim to maximize the $\ell_2$ norm of their output connections. However, in the CNN scenario, each filter has multiple connections to the filters in the subsequent layer. The limitation on weight updates is also imposed by the $\ell_2$ norm of changes applied to the input connections of the filter. Notations and the \algname{} algorithm for CNNs are presented in Appendix \ref{Notation:Appendix}, \ref{Algorithms:Appendix}. To establish the Nash equilibrium property of \algname{} as $\alpha$ approaches zero in the CNN setting, we present the following theorem, which serves as the CNN counterpart to Theorem \ref{semi_is_nash_label}.

\begin{theorem}
\label{semi_is_nash_label_CNN}
Consider a CNN with  piecewise linear activation functions. Where all $\alpha^{[l]}_j$ are equal to $\alpha$, and all filters in layers after the $i$-th layer employ \algname{}:\\
1) $\exists c,j: \sum_{s=1}^{m}{\sum_{p=1}^{D}r^{[l]}_{s,i,p}a^{[l-1]}_{s,c,p-j}} \neq 0$,\\
2) $g^{[l+1]}$ is independent of filter $i$'s response at iteration $\tau-1$,\\ 
3) $\forall_{s, j}: z^{[l]}_{s,i,j}\notin f^{[l]} \text{breakpoints}$,
then:\\
\begin{small}
1)~$
\lim_{\alpha \to 0^+} \text{normalize}({{b^{[l]}_i}^*(\tau-1)}) = \text{normalize}({b^{[l]}_i(\tau-1)})
$,\\
2)~$
\lim_{\alpha \to 0^+} \frac{\normx{{b^{[l]}_i}^*(\tau-1)}}{\normx{b^{[l]}_i(\tau-1)}} = 1$.\\
\end{small}
\normalsize
Here,  \textit{normalize} signifies an operation that scales the matrix such that its Frobenius norm becomes equivalent to 1.
\end{theorem}
The role of the assumptions in this theorem parallels that of Theorem \ref{semi_is_nash_label}. 
In the following, we only consider the MLP case, and CNN discussion is left for the Appendix. 

Let us now delve into the reasons why the responses of all neurons tend to converge to the \algname{} response when the value of $\alpha$ approaches zero and all the assumptions stated in Theorem \ref{semi_is_nash_label} are satisfied. As Theorem \ref{semi_is_nash_label} is solely conditioned on the algorithm employed by the next layer neurons, we can deduce that the responses of the neurons in the last layer will converge to \algname{}, irrespective of the algorithms employed by the other neurons. Similarly, the responses of the neurons in the second-to-last layer will converge to \algname{} only if the last layer neurons employ \algname{}. However, based on the previously mentioned statement, we can prove that the responses of the last layer neurons unconditionally converge to \algname{} as $\alpha$ approaches zero. Consequently, the responses of the neurons in the next-to-last layer will also converge to the \algname{} response. We can iteratively apply a similar reasoning for each layer, proceeding backward layer by layer, to establish the convergence of neuron responses. This is formally demonstrated in Theorem \ref{convergence_theorem}.

\begin{theorem}
\label{convergence_theorem}
Consider an MLP with neurons using a piecewise linear activation functions. Assume that for all $i$ and $l$ the following conditions hold:\\
1) $0 < \normx{{r^{[l]}_{:i}}^\top.a^{[l-1]}}$,\\
2) $g^{[l+1]}$ is independent of neuron $i$'s response at iteration $\tau-1$,\\
3) $\forall_{s}: z^{[l]}_{s,i}\notin f^{[l]} \text{breakpoints}$.
Then, the responses of all neurons converge to the \algname{} response as follows:\\
\begin{small}
1) $\lim_{\alpha \to 0^+} \text{normalize}({{b^{[l]}_i}^*(\tau-1)}) = \text{normalize}({b^{[l]}_i(\tau-1)})$,\\
2) $\lim_{\alpha \to 0^+} \frac{\normx{{b^{[l]}_i}^*(\tau-1)}}{\normx{b^{[l]}_i(\tau-1)}} = 1$,\\
\end{small}
for every neuron $i$ in layer $l$, and every layer $l$ in the MLP, where \textit{normalize} refers to $\ell_2$ vector normalization operator. 
\end{theorem}

Theorem~\ref{convergence_theorem} implies that \algname{} converges to the best response of each neuron as $\alpha$ approaches zero. More importantly, this theorem suggests that if neurons utilize a first-order approximation for their utility for larger values of $\alpha$ to update their weights, \algname{} represents their response.

\section{Experiment}
\label{experiment_section}
The primary objective of this section is to evaluate and compare the performance of the \algname{} algorithm with error-backpropagation in the context of deep learning. To achieve this, we present experimental results based on three popular benchmark datasets: MNIST \cite{mnist}, Split-MNIST \cite{zenke2017continual}, and CIFAR-10 \cite{cifar}. Moreover, the details of datasets, data pre-processing, and models' architecture are provided in  Appendix \ref{Experiments:Appendix}.

\begin{figure*}[t]
    \centering
    \subfigure[]{\includegraphics[width=0.534\textwidth]{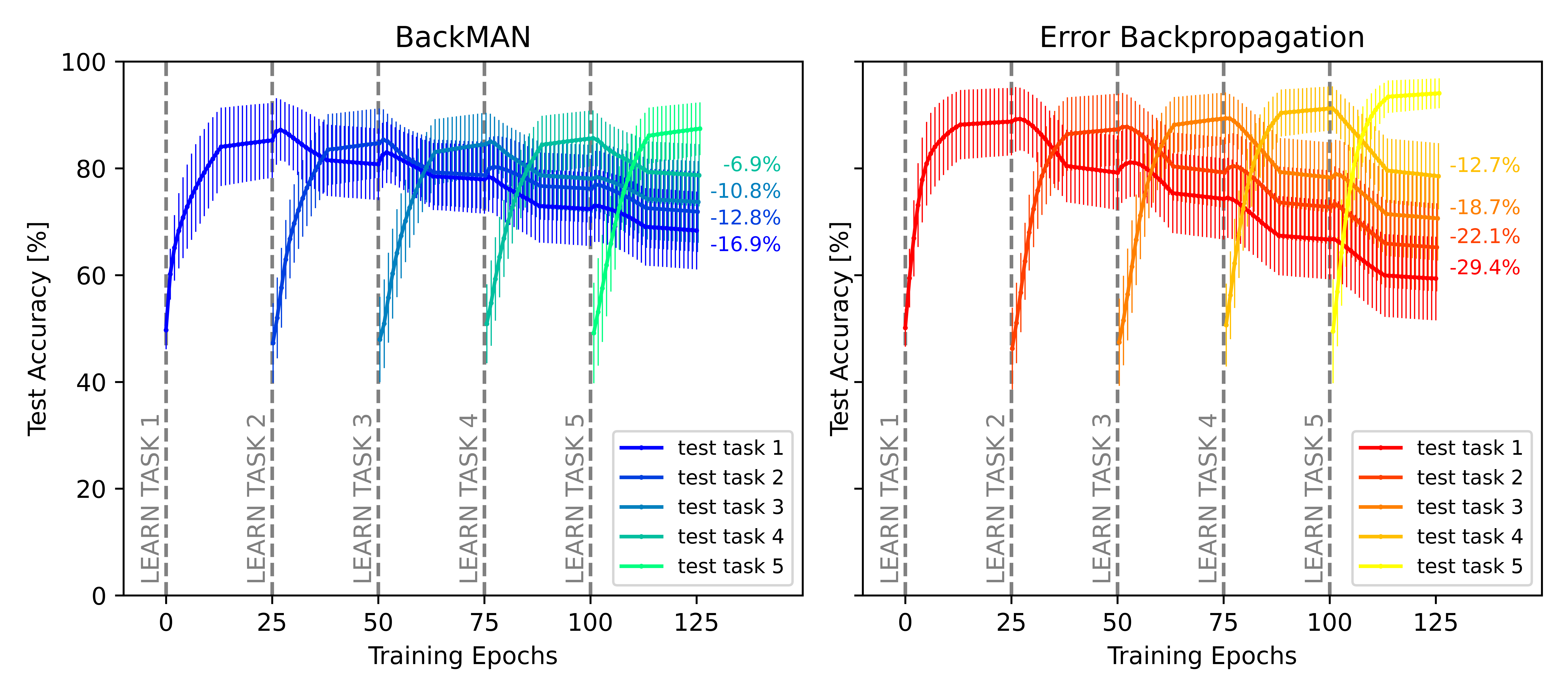}} 
    \subfigure[]{\includegraphics[width=0.273\textwidth]{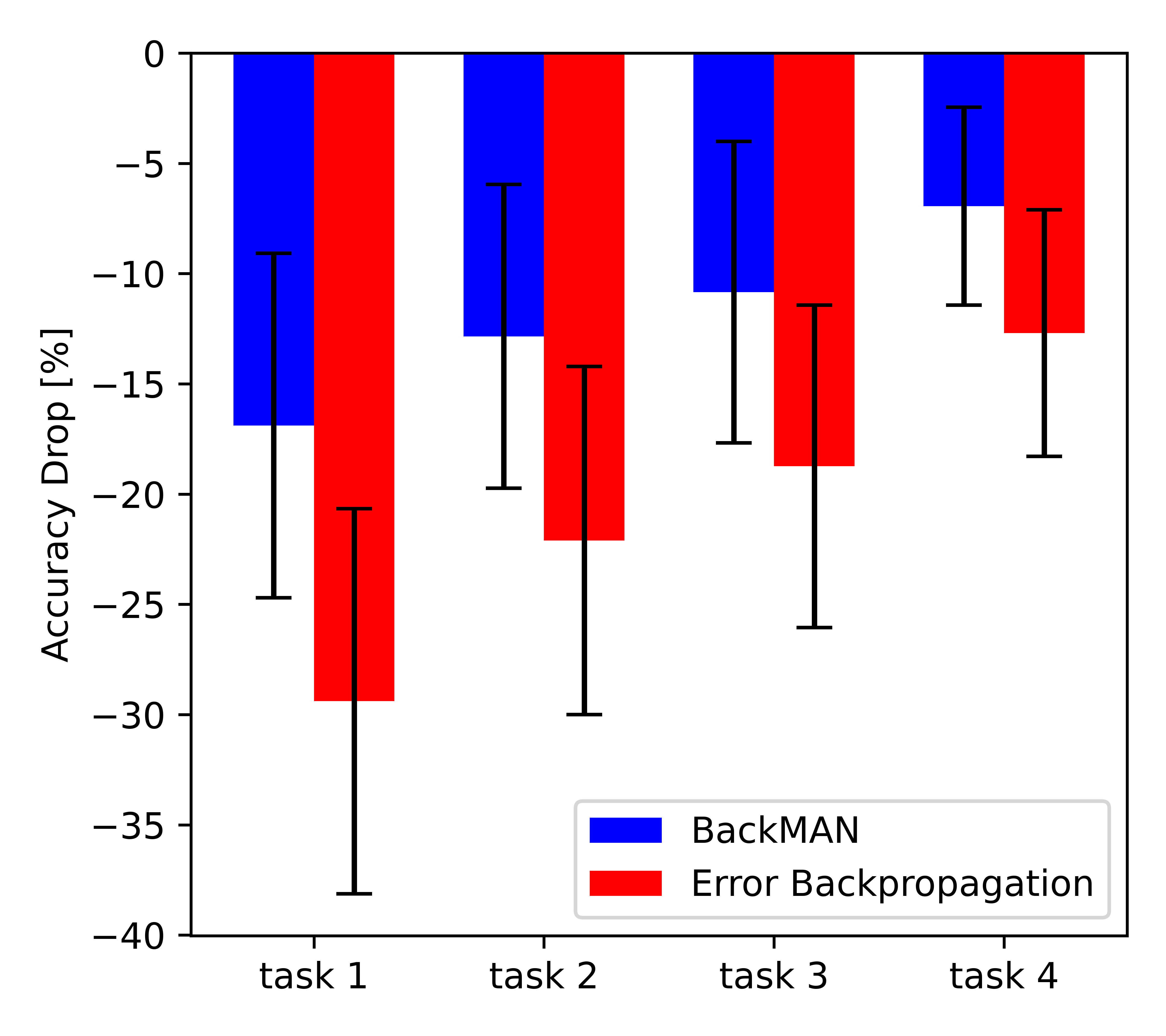}} 
    \caption{
    (a) Task-specific accuracies during training for both error-backpropagation and \algname{}. Numbers on the right side of the plots indicate the accuracy difference between the end of training and the completion of the respective task. (b) Accuracy difference between the end of the overall training and the end of the task of interest's training.}
    \label{fig:forgetting}
\end{figure*}

\subsection{Classification}
\label{experiment_classification}

To demonstrate that the \algname{} algorithm can achieve results comparable to error-backpropagation, we tested both algorithms on the MNIST, CIFAR-10, and CIFAR-100 datasets.
According to the results in Table~\ref{ClassificationTabel}, \algname{} can attain almost the same accuracy on the MNIST dataset and CIFAR-10 with ResNet-18 and only lower accuracy on the CIFAR-10 (LeNet-5) and CIFAR-100 datasets while maintaining significantly higher accuracy compared to other baselines. 

For the MNIST experiment, an MLP model is used. The model consists of 4 layers consist of 2000 neurons connected to each other with a ReLU activation function and the last layer has 10 neurons and uses a softmax activation to produce probabilities. For the CIFAR-10 (Le-Net5) experiment, a Le-Net5 model is used with two sets of convolutional and max pooling layers, followed by two fully connected layers and an output layer that uses Softmax. In ResNet-18 experiments, the standard model used but the initial convolution layer modified to compensate the low resulution of CIFAR-10 and CIFAR-100 datasets.

For training the models, 
we conducted a grid search with batch sizes of 64, 128, 256, and 512, as well as learning rates/$\alpha$ of 0.002, 0.01, 0.05, and 0.25 to identify the optimal hyperparameters. 
Next, the models are trained using the identified hyper-parameters for 100 epochs with 5 different seeds. 
We evaluated the accuracy of each model using the validation set after every epoch, selecting the best-performing model, which might not necessarily be the model after the final epoch.

In Table \ref{ClassificationTabel}, we also present the results of DropConnect \cite{wan2013} and Adaptive DropConnect, as reported in \cite{AuctionNeuron}. However, we acknowledge that comparing these methods may not be fair due to differing assumptions regarding neurons' action space, neural observations, and underlying mechanisms.

\subsection{Catastrophic Forgetting}

Catastrophic forgetting, a widely recognized phenomenon in the machine learning domain, occurs when models inadvertently lose previously acquired knowledge upon training for a new task \cite{forgetting1, forgetting2}. In contrast, humans and biological neural networks generally do not exhibit this behavior. To demonstrate that our model emulates properties expected from a biological system, which has been represented by multi-agent systems in various studies \cite{ott2020giving, balduzzi2014kickback, Lewis2014,wang2015reinforcement}, we assess our method in a catastrophic forgetting benchmark with the Split-MNIST dataset. As illustrated in Table \ref{forgetting tabel}, our method outperforms error-backpropagation by achieving a higher average accuracy.

Fig. \ref{fig:forgetting}(a) depicts task-specific accuracies, 
suggesting that although \algname{} may attain marginally lower 
accuracy on individual tasks, it maintains its performance when 
trained on subsequent tasks, culminating in 
superior average accuracy. In Fig. \ref{fig:forgetting}(b), we show the accuracy decline between the final training epoch and the last epoch in that models are trained on the task of interest. It is apparent that \algname{} has half the accuracy reduction compared to error-backpropagation across all tasks.

For this experiment, we adhered to the same hyperparameter grid search as delineated in Section \ref{experiment_classification}. The models were evaluated based on their final average accuracy across all five tasks. Each model underwent 25 epochs of training for every task. To augment our model's stability, we employed a dynamic selection of the $\alpha$ value, ensuring that neurons within each layer shared an equal value of $\alpha$ to preserve fairness in their competitive interactions. The $\alpha$ value for each layer was determined such that the Frobenius norm of the $g$ vector is equal to the Frobenius norm of $\delta$ that would be produced if that layer was trained with gradient descent. 

\subsection{Depth Analysis}
\begin{figure}[ht]
    \centering
\label{fig:depth}
\includegraphics[width=0.38\textwidth]{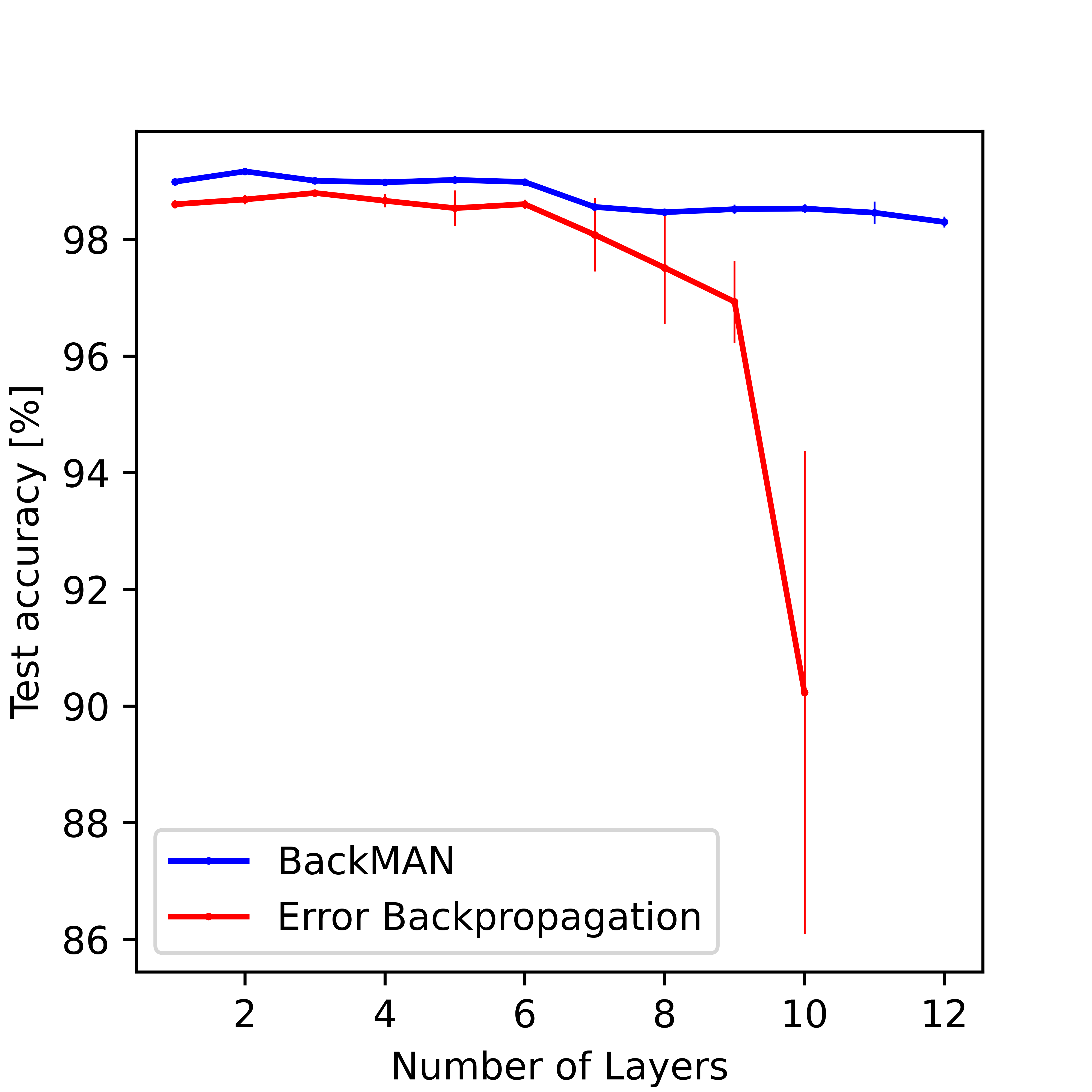}
\caption{Performance comparison of the \algname{} algorithm and Error-Backpropagation across varying network depths, with each layer consisting of 2000 neurons on MNIST dataset. Mean and standard deviations are calculated using five seeds. }
\end{figure}
In this Section, we endeavor to assess the performance of the \algname{} algorithm across a diverse range of neural network depths. The evaluation is conducted on MLPs with hidden layer counts ranging from one to twelve, each consisting of $2000$ neurons. The dataset employed for this analysis is MNIST dataset. Our findings, as depicted in Fig. \ref{fig:depth}, reveal that the \algname{} algorithm surpasses the performance of error-backpropagation in network configurations with more than six layers. This superior performance can potentially be attributed to the normalization-like effects that \algname{} exerts on the signals propagated by neurons to preceding layers. We furture investigate this normalization-like effect of \algname{} in Appendix.~\ref{appendix:DampingEffect}. Additional results for the CIFAR-100 dataset and ResNet50 model are also presented in Appendix.~\ref{appendix:depthanalysis}. 

For this experiment, a batch size of 1024 and a grid search to find the optimal learning rate is utilized. The search is conducted over a set of candidate learning rates/$\alpha$, namely $0.0004$, $0.002$, $0.01$, $0.05$, and $0.25$. Next, models are trained using optimal hyperparameters with five different seeds. The performance of the best model, identified using a validation set, of each run is subsequently evaluated and averaged across the test set.

\subsection{Performance}

A prevalent drawback of multi-agent systems is their inability to leverage the acceleration offered by GPUs or TPUs.
The primary reason for this limitation is that many of these algorithms do not rely on gradients, unlike conventional deep-learning training approaches. Libraries designed for deep learning are typically constructed in a manner that allows efficient gradient calculations. Therefore these multi-agent systems can not utilize these libraries. However, we have implemented our method in such a way that it can be executed on a GPU. To achieve this, we used the PyTorch library \cite{PyTorch} and applied implementation techniques that enabled the training of our network, despite not using gradients and the library being designed for gradient-based model training. We provide further details in Appendix \ref{ImplementationAppendix}.

To compare the runtime of the proposed training algorithm, we performed 10,000 forward and backward passes with 
a batch size of 1024 using a blank image to eliminate the influence of data loaders. In all of our experiments, including this section, we employed a single P100 GPU. 
Table.~\ref{performanceTabel} depicts the results, revealing our method is roughly two times slower than EBP.

\begin{table}
\caption{Runtimes for 10,000 forward and backward passes of batches with 1024 blank images.}
  \label{performanceTabel}
    \centering
  \begin{tabular}{lccl}
    \toprule
    Model     & MLP     & LeNet-5 \\
    \midrule
    EBP & 117s & 61s \\
    \algname{} & 263s & 129s \\
    \bottomrule
  \end{tabular}
\end{table}
\section{Conclusion}

The motivation for utilizing multi-agent neural networks is compelling; however, they often lack scalability, particularly when compared to the remarkable capabilities of error-backpropagation for training deep neural networks. In this work, we propose a local objective for neurons that, if pursued greedily by the neurons, the network emulates behaviors similar to error-backpropagation. We first establish that if neurons optimally respond in the suggested setting, their behavior converges to that of our proposed algorithm. Subsequently, we demonstrate that the algorithm can achieve impressive results, particularly in a catastrophic forgetting benchmark and when considering deep MLPs.

\section*{Acknowledgments}
The authors would like to thank Maral Jabbari Shiviari, John Lazarsfeld, and Mher Safaryan for their advice on aspects of the project.

\bibliography{arXiv}
\bibliographystyle{icml2023}
\newpage
\onecolumn
\newpage
\appendix

\addcontentsline{toc}{section}{Appendix} 

\section{Method Theoretical parts}
\label{theoreticalAppendix}

\subsection{Notation for Convolutional Neural Networks (CNNs)}
\label{Notation:Appendix}

We focus on CNNs operating on one-dimensional (1D) data, such as time series, to streamline our discussion. However, it is important to note that our proofs and discussions can be generalized to two-dimensional (2D) data, such as images. In this case, the input is represented as $X\in\mathds{R}^{m\times n_x \times d_0}$.

Unlike MLPs, CNNs do not utilize the concept of neurons but rather employ filters, which we treat as self-interested agents. For a given CNN, we denote the filters as $w^{[l]}\in\mathds{R}^{n_l\times n_{l-1}\times k_l}$, where $k_l$ signifies the kernel size. The output from layer $l$ is defined as $a^{[l]}\in\mathds{R}^{m\times n_l\times d_l}$, where $d_l$ denotes the dimension of each filter's output. Hence, we can express

\begin{equation}
\label{forward_pass:CNN}
a^{[l]}_{s,i,p}=f^{[l]}(\sum_{j=1}^{n_{l-1}} \sum_{e=1}^{k_l} a^{[l-1]}_{s,j,p-e+k_l}.w^{[l]}_{i,j,e}).
\end{equation}

Subsequently, this implies that $d_l=d_{l-1}-k_l+1$. To make our equations more compact, we introduce the \textbf{Conv} notation, denoting the convolution operation. Thus,

\begin{equation}
\label{forward_pass:CNN:short}
a^{[l]}=f^{[l]}(\conv{a^{[l-1]}}{w^{[l]}}).
\end{equation}
The specifics of this operation are explicated comprehensively in Table~\ref{CNN_notation:table}.

Furthermore, the objective function for filter $i$ in layer $l$ is defined as follows:
\begin{equation}
\label{utility_equation:CNN}
    U^{[l]}_{i}(\tau)=\normf{w^{[l+1]}_{:i:}}=\sqrt{\sum_{j,k} (w^{[l+1]}_{j,i,k}(\tau))^2}.
\end{equation}
In addition, the update limit in Eq.~\eqref{updateAssumption} modifies to:
\begin{equation}
    \normf{w^{[l]}_{i::}(\tau)-w^{[l]}_{i::}(\tau+1)} \leq \alpha^{[l]}_i(\tau).
    \label{updateAssumption:CNN}
\end{equation}

\subsection{\algname{} Algorithm for CNN}
\label{Algorithms:Appendix}
In this subsection, we provide an exposition of \algname{} algorithm for Convolutional Neural Networks (CNNs) utilizing a piecewise linear function as the activation function between layers.

Given that we are dealing with one-dimensional data, variables $g$, $r$, $a$, and $w$ are not two-dimensional matrices. Instead, they are three-dimensional.

So weight update is performed as follows:
\begin{equation}\label{eq:update_rule:CNN}
    w^{[l]}_{i,j,e}(\tau+1) = w^{[l]}_{i,j,e} + \sum_{s,p}^{m,d_l-k_l+1} g^{[l]}_{s,i,p}.a^{[l-1]}_{s,j,p-e+k_l}.
\end{equation} 
Considering the frequency of this operator's use, we define it as the operator $\psi$ to simplify this notation. The detailed notation is stipulated in Table~\ref{CNN_notation:table}. This operator essentially performs the calculations mentioned above for all weights. 

Using $\psi$ we can rewrite Eq.~\eqref{eq:update_rule:CNN} as:

\begin{equation}\label{eq:update_rule:CNNshort}
    w^{[l]}(\tau+1) = w^{[l]} +\rconv{g^{[l]}}{a^{[l-1]}}.
\end{equation} 

Furthermore, the $g^{[l]} \in \mathds{R}^{m \times n_l \times d_l}$ variable is defined as follows:
\begin{equation}
\label{genral_g:CNN}
    g^{[l]}_{s,i,p} \coloneqq  r^{[l]}_{s,i,p}.c^{[l]}_{i},
\end{equation}
where $c^{[l]} \in \mathds{R}^{n_l}$ comes as the normalization factor that enforces \eqref{updateAssumption:CNN} in \eqref{eq:update_rule:CNN}. Based on the objective functions defined in subsection~\ref{Notation:Appendix},the $r^{[L]}$ variable for the last layer is calculated as follows:
\begin{equation}
\label{lastLayer_R:CNN}
    r^{[L]}  \coloneqq  -\frac{\partial \ell}{\partial a^{[L]} }.
\end{equation}
For other layers, $r^{[l]}$ is calculated based on $g^{[l+1]}$ as follows:
\begin{equation}
\label{hiddenLayer_R:CNN:indentity}
    \forall_{s\leq m, i \leq n_l, p \leq d_l}: r^{[l]}_{s,i,p}=\sum_{j,e}^{n_{l+1}, k_l} g^{[l+1]}_{s,j,p+e-k_l}.w^{[l+1]}_{j,i,e}.{f^{[l]}}^{'}(a^{[l]}_{s,i,p}).
\end{equation}
If we use a cell in the summation that does not exist, e.g., $g^{[l+1]}_{1,1,-1}$, that cell value is assumed to be zero. To simplify this notation, we use the notation $\text{ConvT}$, which calculates this value for the whole layer. In the literature, this function is usually called "convolution transposed." The exact definition is provided in Table~\ref{CNN_notation:table}.

Using this definition we can simplify Eq.~\eqref{hiddenLayer_R:CNN:indentity} as:
\begin{equation}
\label{hiddenLayer_R:CNN}
r^{[l]}_{s,i,p}=\deconv{g^{[l+1]}}{w^{[l+1]}} \odot {f^{[l]}}^{'}(a^{[l]})
\end{equation}
where ${f^{[l]}}^{'}(x)$ refers to derivative of ${f^{[l]}}(x)$.
Moreover, by substituting Eq.~\eqref{eq:update_rule:CNN} into Eq.~\eqref{updateAssumption:CNN}, we obtain:
\begin{equation}
\label{general_c:CNN}
    c^{[l]}_i \leq \frac{\alpha^{[l]}_i}{\sqrt{\sum_{j,e} ({\sum_{s,p} r^{[l]}_{s,i,p}.a^{[l-1]}_{s,j,p-e}})^2 }}=\frac{\alpha^{[l]}_i}{\normf{\rconv{r^{[l]}}{a^{[l-1]}}_{i}}}
\end{equation}

We define $c^{[l]}_i$ to set it to this upper bound.  Thus, the CNN algorithm is as follows:
\begin{algorithm}
\caption{One Iteration in \algname{} Algorithm for CNN}  \label{alg:semiGD:CNN}
\begin{algorithmic}[0]
\State $a^{[0]} \gets X$ \Comment{forward pass}
\For{l:=1}{L} 
    \State $a^{[l]} \gets f^{[l]}(\conv{a^{[l-1]}}{w^{[l]}})$
\EndFor
\\    
\State $r^{[L]} \gets -\frac{\partial \ell}{\partial a^{[L]}}$ \Comment{backward pass}
\\
\For {l:={L}}{1}
    \State $\forall_{i\leq {n_l}}: c^{[l]}_i \gets 
    \frac{\alpha^{[l]}_i}{\normf{\rconv{r^{[l]}}{a^{[l-1]}}_{i}}}$\\ 
    \State \ \ \ \ \ \ $\forall_{s,i,p}: g^{[l]}_{s,i,p} \gets r^{[l]}_{s,i,p}.c^{[l]}_i$ 
    \State \ \ \ \ \ \  $r^{[l-1]} \gets \deconv{g^{[l]}}{w^{[l]}} \odot {f^{[l]}}^{'}(a^{[l]})$ 
\EndFor
\\
\For{l:=1}{L} \Comment{weight update}
    \State $w^{[l]} \gets w^{[l]}+\rconv{g^{[l]}}{a^{[l-1]}}$
\EndFor
\end{algorithmic}
\end{algorithm}

The functional components that are integral to Convolutional Neural Networks (CNNs) are comprehensively collated in Table~\ref{CNN_notation:table}.
\begin{table}[h]
\caption{Notation for CNN-related functions.}
\centering
\begin{adjustbox}{scale=0.90}
\renewcommand{\arraystretch}{2}
\begin{tabular}{cccc}
\hline
Notation & Type & Definition\\
\hline
$\deconv{.}{.}$ & $\mathds{R}^{m\times n' \times d}\times\mathds{R}^{n'\times n \times k}\to\mathds{R}^{m\times n \times d+k-1}$ &  $\deconv{A}{B}_{s,i,p}=\sum_{j,e}^{n', k} A_{s,j,p+e-k}B_{j,i,e}$ \\
$\rconv{.}{.}$  & $\mathds{R}^{m\times n' \times d}\times\mathds{R}^{m\times n \times d+k-1}\to\mathds{R}^{n'\times n \times k}$ & $\rconv{A}{B}_{i,j,e}=\sum_{s,p}^{m,d} A_{s,i,p}B_{s,j,p-e+k}$\\
$\conv{.}{.}$   & $\mathds{R}^{m\times n' \times d}\times\mathds{R}^{n\times n' \times k}\to\mathds{R}^{m \times n \times d-k+1}$ & $\conv{A}{B}_{s,i,p}=\sum_{j,e}^{n',k} A_{s,j,p-e+k}B_{i,j,e}$
\\
$\langle.,.\rangle_F$ & $\mathds{R}^{a_1\times ... \times a_k} \times \mathds{R}^{a_1\times ... \times a_k}\to\mathds{R}$ & $\langle A,B\rangle_F = \sum_{e\in \text{all elements}}A_eB_e$ \\
$.\odot.$  & $\mathds{R}^{a_1\times ... \times a_k}\times \mathds{R}^{a_1\times ... \times a_k} \to\mathds{R}^{a_1\times ... \times a_k}$ & $(A\odot B)_{e\in \text{all elements} }=A_{e}B_{e}$  \\
\hline
\end{tabular}
\end{adjustbox}
\label{CNN_notation:table}
\end{table}

\subsection{Relation to Gradient Descent}
\label{relation to gradient descent:Appendix}
In this section, we aim to provide the proof for theorems in the section~\ref{Relation to gradient descent:label}. Those theorems aim to establish a link between the \algname{} algorithm and traditional gradient descent under certain conditions. To do so, we first introduce some notation. Let $\delta^{[l]}$ denote $\frac{\partial \ell}{\partial z^{[l]}}$, the partial derivative of the loss function $\ell$ with respect to the pre-activations $z^{[l]}$ at layer $l$. Before proving the Theorem, we try to prove, under the same condition as Theorem~\ref{semi_equal_GD_label}, $g^{[l]} = -\delta^{[l]}$ for all $l$.

\begin{lemma}
\label{g_equal_delta_lemma}
Assuming that $\forall_{i,l}: c^{[l]}_i = 1$ and all neurons use  a piecewise linear function as their activation function, the \algname{} algorithm results in $\forall_{l}: g^{[l]} = -\delta^{[l]}$.
\end{lemma}
\begin{proof}[Proof of Lemma~\ref{g_equal_delta_lemma}]
The proof is conducted using induction.
For the base case, considering the last layer $L$, we have from Eq.~\eqref{genral_g} and Eq.~\eqref{lastLayer_R}:
\begin{equation}
g^{[L]}=-\frac{\partial \ell}{\partial z^{[L]} }. diag(c^{[L]})
\end{equation}
By substituting $c^{[L]}_i=1$ from the lemma assumption, we get
\begin{equation}
    g^{[L]}=-\frac{\partial \ell}{\partial z^{[L]} } 
\end{equation}
which establishes the base case.
For the inductive step, by assuming:

\begin{equation}
g^{[l+1]} = -\delta^{[l+1]},
\end{equation}

We should prove:

\begin{equation}
g^{[l]} = -\delta^{[l]}.
\end{equation}

Using Eq.~\eqref{hiddenLayer_R} and Eq.~\eqref{genral_g}, we obtain 
\begin{equation}
g^{[l]}=((g^{[l+1]}.w^{[l+1]})\odot M^{[l]}).diag(c^{[l]})
\end{equation} 

Applying the lemma assumption $\forall_{i,l}: c^{[l]}_i = 1$, we get: 
\begin{equation}
\label{circ_eq:appnedix}
g^{[l]}=(g^{[l+1]}.w^{[l+1]})\odot M^{[l]}=-(\delta^{[l+1]}.{(w^{[l+1]})}))\odot M^{[l]}=-\delta^{[l]}
\end{equation}
which completes the induction. 
\end{proof}

With Lemma~\ref{g_equal_delta_lemma} in place, we can now prove Theorem~\ref{semi_equal_GD_label}.

\textbf{Theorem 1.}
If \ $\forall_{i,l}: c^{[l]}_i = 1$ and all neurons use a piecewise linear function as their activation function, then using algorithm \algname{} $\forall_{l,i,j}: w^{[l]}_{i,j}(\tau+1)=w^{[l]}_{i,j} - \frac{\partial \ell}{\partial w^{[l]}_{i,j}}$.

\begin{proof}[Proof of Theorem~\ref{semi_equal_GD_label}]
We start with the right-hand side of Theorem~\ref{semi_equal_GD_label}.
\begin{equation}
    w^{[l]}_{i,j} - \frac{\partial \ell}{\partial w^{[l]}_{i,j}} = w^{[l]}_{i,j} - \frac{\partial  z^{[l]}_{:i}}{\partial w^{[l]}_{i,j}}.\frac{\partial \ell}{\partial z^{[l]}_{:i}} = w^{[l]}_{i,j} - \frac{\partial  z^{[l]}_{:i}}{\partial w^{[l]}_{i,j}}.\delta^{[l]}_{:i} =w^{[l]}_{i,j} - {a^{[l-1]}_{:j}}^\top.\delta^{[l]}_{:i}.
\end{equation}
By using Lemma~\ref{g_equal_delta_lemma}, this is equivalent to:
\begin{equation}
    w^{[l]}_{i,j} + {a^{[l-1]}_{:j}}^\top.g^{[l]}_{:i}= 
w^{[l]}_{i,j} +{g^{[l]}_{:i}}^\top.{a^{[l-1]}_{:j}}=
w^{[l]}_{i,j} +{({g^{[l]}}^\top.a^{[l-1]})}_{i,j}=
w^{[l]}_{i,j}(\tau+1).
\end{equation}
This concludes the proof.
\end{proof}

\subsection{Nash Equilibrium}
In this section, we provide proofs for the theorems articulated in section~\ref{Nash Equilibrium:label}. We intend to demonstrate that the \algname{} algorithm exhibits Nash Equilibrium-like properties as the parameter $\alpha$ approaches zero. Theorem~\ref{semi_is_nash_label} explores the behavior of an individual neuron as subsequent layer's neurons implement the \algname{} algorithm to adjust their weights. Similarly, Theorem~\ref{convergence_theorem} elucidates the convergence of all neurons' behavior to \algname{} under specified conditions. We also study a single filter's behavior in the context of convolutional neural networks (CNN) as in Theorem~\ref{semi_is_nash_label_CNN}, and furnish Theorem~\ref{convergence_theorem:CNN} on the behavior of all filters, analogous to Theorem~\ref{convergence_theorem}.

\subsubsection{Proofs of Theorem~\ref{semi_is_nash_label} and Theorem~\ref{convergence_theorem}}
Before proving Theorem~\ref{semi_is_nash_label}, we need two important lemmas.

\begin{lemma} 
\label{g_is_in_order_alpha}
Suppose we have a Multi-Layer perceptron (MLP) with neurons utilizing a piecewise linear function with \algname{} algorithm. Then we have:  $\forall_l: g^{[l]} \propto \alpha$, and $\forall_{l\neq L}: g^{[l]} \propto \alpha$
\end{lemma}
\begin{proof}[Proof of Lemma~\ref{g_is_in_order_alpha}]
The proof is conducted using induction. For the base case, considering the last layer $L$, from Eq.~\eqref{lastLayer_R}, we know that $r^{[L]}$ is independent of $\alpha$. Given that Eq.~\eqref{general_c} implies $c^{[L]}\propto \alpha$, by applying Eq.~\eqref{genral_g}, we conclude that $g^{[L]} \propto \alpha$.

For the inductive step,  assuming $g^{[l+1]} \propto \alpha$ we want to prove $g^{[l]} \propto \alpha$ and $r^{[l]} \propto \alpha$. Using Eq.~\eqref{hiddenLayer_R} and knowing that $w^{[l+1]}$ is not a function of $\alpha$, we conclude that $r^{[l]} \propto \alpha$. By definition, $\propto$ implies the existence of  ${r^{[l]}}^*$ such that $r^{[l]}(\alpha) = \alpha.{r^{[l]}}^*$.

Using Eq.~\eqref{general_c}, we deduce:
\begin{equation}
    c^{[l]}_i = \frac{\alpha^{[l]}}{\normx{{r^{[l]}_{:i}}^\top.a^{[l-1]}}}=\frac{\alpha^{[l]}}{\normx{
    {{r^{[l]}_{:i}}^*}^\top.a^{[l-1]}\alpha}}=\frac{1}{\normx{
    {{r^{[l]}_{:i}}^*}^\top.a^{[l-1]}}}.
\end{equation}
This implies that $c^{[l]}$ is not a function of $\alpha$. Therefore
\begin{equation}
    g^{[l]}_i=r^{[l]}_i.c^{[l]}_i\propto \alpha,
\end{equation}
which proves the induction step and completes the proof.
\end{proof}

\begin{lemma} 
\label{changes_are_small_lemma}
Suppose we have an MLP with neurons using  a piecewise linear function activation and \\$\forall_{l,j,s}: z^{[l]}_{s,j}\notin f^{[l]} \text{breakpoints}$. Then we have:
\begin{equation}\begin{aligned}
\forall_{l,i}: \lim_{\alpha\to0^+}:{g^{[l]}_{:i}}^\top.a^{[l-1]}={g^{[l]}_{:i}}^\top(\tau-1).a^{[l-1]}(\tau-1)
\end{aligned}\end{equation}
\end{lemma}
\begin{proof}[Proof of Lemma~\ref{changes_are_small_lemma}]
We first demonstrate that $\forall_{l}: \lim_{\alpha\to0^+}: M^{[l]}(\tau)=M^{[l]}(\tau-1)$.
This is because  $\forall_{l,j,s}: z^{[l]}_{s,j}\notin f^{[l]} \text{breakpoints}$. If we have an $\alpha<\min_{l,j,s, e \in f^{[l]}\text{breakpoints} } |z^{[l]}_{s,j}-e|$, then with such $\alpha$, none of the neurons change their linear piece in their piecewise linear function activation.

In addition, when $\alpha$ approaches zero $\forall_l: w^{[l]}(\tau)$ converges to $w^{[l]}(\tau-1)$, and considering that the dataset remains unchanged during iterations, we establish that $\forall_{l}: \lim_{\alpha\to0^+}:a^{[l]}=a^{[l]}(\tau-1)$. 
We can then deduce the same statement for $g$; namely, $\forall_{l}: \lim_{\alpha\to0^+}:g^{[l]}=g^{[l]}(\tau-1)$. 
Combining these results validates the lemma.
\end{proof}

\textbf{Theorem 2.}
Considering an MLP with neurons using a piecewise linear function activation, if $\forall_{l,j}: \alpha^{[l]}_j=\alpha$, and all neurons in layers after the $i$-th layer employ \algname{}, $0 < \normx{{r^{[l]}_{:i}}^\top.a^{[l-1]}}$, $g^{[l+1]}$ is independent of neuron $i$'s response at iteration $\tau-1$ and $\forall_{l,j,s}: z^{[l]}_{s,j}\notin f^{[l]} \text{breakpoints}$ then:\\
1)~$
\lim_{\alpha \to 0^+} \text{normalize}({{b^{[l]}_i}^*(\tau-1)}) = \text{normalize}({b^{[l]}_i(\tau-1)})
$,\\
2)~$
\lim_{\alpha \to 0^+} \frac{\normx{{b^{[l]}_i}^*(\tau-1)}}{\normx{b^{[l]}_i(\tau-1)}} = 1 
$. Where \textit{normalize} refers to $\ell_2$ vector normalization operator. 

\begin{proof}[Proof of Theorem~\ref{semi_is_nash_label}]
First, we prove the Theorem for hidden layers' neurons. We want to prove neuron $i$ in layer $l\neq L$ best strategy is converging to \algname{} under the Theorems assumptions. Suppose this neuron modifies its input weights with vector $b^{*} \in \mathds{R}^{1 \times n_{l-1}}$, thus  $w^{[l]}_i=w^{[l]}_i(\tau-1)+b^{*}$. The action of this neuron impacts its utility function in iteration $\tau+1$. Given our assumption of rational and greedy neurons, it will attempt to optimize its utility function in iteration $\tau+1$. By Eq.~\eqref{utility_equation:body} we have,

\begin{equation}
U^{[l]}_i(\tau+1) =\normx{w^{[l+1]}_{:i}(\tau+1)}.
\end{equation}
Knowing that subsequent neurons employ \algname{}, and using Eq.~\eqref{eq:update_rule}, the value is equivalent to:
\begin{equation}
\normx{{w^{[l+1]}}_{:i}+{g^{[l+1]}}^\top.{a^{[l]}}_{:i}}.
\end{equation}
Conversely, considering the neuron's rationality and greediness, we have
\begin{equation}
U^{[l+1]}_i(\tau+1) = \max_{b} \normx{w^{[l+1]}_{:i}+{g^{[l+1]}}^\top.a^{[l]}_{:i}},
\end{equation}
where $\normx{b}\leq \alpha$. Since $U$ is always positive, maximizing $U$ will yield the same result as maximizing $U^2$. So we investigate the latter,
\begin{equation}
(U^{[l+1]}_i(\tau+1))^2 = \max_{b} \normx{w^{[l+1]}_{:i}+{g^{[l+1]}}^\top.a^{[l]}_{:i}}^2
\end{equation}
\begin{equation}
= \max_{b} \normx{w^{[l+1]}_{:i}}^2+\normx{{g^{[l+1]}}^\top.a^{[l]}_{:i}}^2+2 {w^{[l+1]}_{:i}}^\top.({g^{[l+1]}}^\top.a^{[l]}_{:i}).
\end{equation}
Since $f^{[l]}$ is a piecewise linear function, and $a^{[l]}_{s,i}$ not a breaking point, $f^{[l]}$ around $a^{[l]}_{s,i}$ behaves as: $f^{[l]}(x)=M^{[l]}_{s,i}.x+o^{[l]}_{s,i}$. which $o$ is the offset of the linear function.

 By applying the forward pass definition from Eq.~\eqref{forward_pass}, this equation can be rewritten as: 
\begin{equation}\begin{aligned}
\label{first_equation_using_forward_pass}
\max_{b} \normx{w^{[l+1]}_{:i}}^2+\normx{{g^{[l+1]}}^\top.(M^{[l]}_{:i} \odot (a^{[l-1]}.{w^{[l]}_{i}}^\top) + o^{[l]}_{:i})}^2\\
+2{w^{[l+1]}_{:i}}^\top.{g^{[l+1]}}^\top.(M^{[l]}_{:i} \odot (a^{[l-1]}.{w^{[l]}_{i}}^\top) + o^{[l]}_{:i}).
\end{aligned}\end{equation}

By defining $z(\tau)\coloneqq {g^{[l+1]}}^\top(\tau).\diag{M^{[l]}_{:i}}.a^{[l-1]}(\tau)$, \eqref{first_equation_using_forward_pass} becomes:

\begin{equation}\begin{aligned}
\max_{b} \normx{w^{[l+1]}_{:i}}^2+\normx{z.{w^{[l]}_{i}}^\top + {g^{[l+1]}}^\top.o^{[l]}_{:i}}^2+2{w^{[l+1]}_{:i}}^\top.z.{w^{[l]}_{i}}^\top\\
+2{w^{[l+1]}_{:i}}^\top.{g^{[l+1]}}^\top.o^{[l]}_{:i}
\end{aligned}\end{equation}

\begin{equation}\begin{aligned}
=\max_{b} \normx{w^{[l+1]}_{:i}}^2+
\normx{z.{w^{[l]}_{i}}^\top}^2+
\normx{{g^{[l+1]}}^\top.o^{[l]}_{:i}}^2+
2.{o^{[l]}_{:i}}^\top.g^{[l+1]}.z.{w^{[l]}_{i}}^\top\\
+2{w^{[l+1]}_{:i}}^\top.z.{w^{[l]}_{i}}^\top + 2 {w^{[l+1]}_{:i}}^\top.{g^{[l+1]}}^\top.o^{[l]}_{:i}
\end{aligned}\end{equation}
\begin{equation}\begin{aligned}
=\max_{b} \normx{w^{[l+1]}_{:i}}^2+
\normx{z.{w^{[l]}_{i}}^\top}^2+
\normx{{g^{[l+1]}}^\top.o^{[l]}_{:i}}^2
+2.{o^{[l]}_{:i}}^\top.g^{[l+1]}.z.{w^{[l]}_{i}(\tau-1)}^\top\\
+2.{o^{[l]}_{:i}}^\top.g^{[l+1]}.z.{b}^\top
+2{w^{[l+1]}_{:i}}^\top.z.{w^{[l]}_{i}}^\top
+ 2 {w^{[l+1]}_{:i}}^\top.{g^{[l+1]}}^\top.o^{[l]}_{:i}
\end{aligned}\end{equation}
Given that $o$ is a constant since $\alpha$ is small and $g^{[l+1]}$ is independent of the choice of $b$, we can decouple some of the terms from the maximum as follows:
\begin{equation}
=2 {w^{[l+1]}_{:i}}^\top.{g^{[l+1]}}^\top.o^{[l]}_{:i}+
\normx{{g^{[l+1]}}^\top.o^{[l]}_{:i}}^2+
2.{o^{[l]}_{:i}}^\top.g^{[l+1]}.z.{w^{[l]}_{i}(\tau-1)}^\top+
\max_{b} \normx{w^{[l+1]}_{:i}}^2+
\normx{z.{w^{[l]}_{i}}^\top}^2+
2.{o^{[l]}_{:i}}^\top.g^{[l+1]}.z.{b}^\top+
+2{w^{[l+1]}_{:i}}^\top.z.{w^{[l]}_{i}}^\top
\end{equation}
If we define $\kappa_1$ as:
\begin{equation}\begin{aligned}
\kappa_1 := 2 {w^{[l+1]}_{:i}}^\top.{g^{[l+1]}}^\top.o^{[l]}_{:i}+
\normx{{g^{[l+1]}}^\top.o^{[l]}_{:i}}^2+2.{o^{[l]}_{:i}}^\top.g^{[l+1]}.z.{w^{[l]}_{i}(\tau-1)}^\top,
\end{aligned}\end{equation}
the equation simplifies to:
\begin{equation}
=\kappa_1+
\max_{b} \normx{w^{[l+1]}_{:i}}^2+
\normx{z.{w^{[l]}_{i}}^\top}^2+
+2{w^{[l+1]}_{:i}}^\top.z.{w^{[l]}_{i}}^\top+
2.{o^{[l]}_{:i}}^\top.g^{[l+1]}.z.{b}^\top
\end{equation}

\begin{equation}
= \kappa_1+\max_{b} \normx{w^{[l+1]}_{:i}}^2+\normx{z.({w^{[l]}_{i}}^\top(\tau-1)+b^\top) }^2+2 {w^{[l+1]}_{:i}}^\top.z.({w^{[l]}_{i}}^\top(\tau-1)+b^\top) +2{o^{[l]}_{:i}}^\top.g^{[l+1]}.z.{b}^\top
\end{equation}
\begin{equation}\begin{aligned}
= \kappa_1+\max_{b} \normx{w^{[l+1]}_{:i}}^2+\normx{z.{w^{[l]}_{i}}^\top(\tau-1) }^2+
\normx{z.b^\top }^2+
2.b.z^\top.z.{w^{[l]}_{i}}^\top(\tau-1)\\
+ 2 {w^{[l+1]}_{:i}}^\top.z.{w^{[l]}_{i}}^\top(\tau-1)
+2 {w^{[l+1]}_{:i}}^\top.z.b^\top+2{o^{[l]}_{:i}}^\top.g^{[l+1]}.z.{b}^\top.
\end{aligned}
\end{equation}
Given the theory assumption that $z$ is independent of the choice of $b$, we can decouple some of the terms from the maximum as follows:

\begin{equation}\begin{aligned}
\kappa_1+\normx{w^{[l+1]}_{:i}}^2+\normx{z.{w^{[l]}_{i}}^\top(\tau-1) }^2+2 {w^{[l+1]}_{:i}}^\top.z.{w^{[l]}_{i}}^\top(\tau-1)\\
+\max_{b}
\normx{z.b^\top }^2+
2.b.z^\top.z.{w^{[l]}_{i}}^\top(\tau-1)+
2 {w^{[l+1]}_{:i}}^\top.z.b^\top+2{o^{[l]}_{:i}}^\top.g^{[l+1]}.z.{b}^\top.
\end{aligned}\end{equation}
If we define $\kappa$ as:
\begin{equation}\begin{aligned}
\kappa := \kappa_1+\normx{w^{[l+1]}_{:i}}^2+\normx{z.{w^{[l]}_{i}}^\top(\tau-1) }^2+2 {w^{[l+1]}_{:i}}^\top.z.{w^{[l]}_{i}}^\top(\tau-1),
\end{aligned}\end{equation}
the equation simplifies to:
\begin{equation}\begin{aligned}
\kappa+
\max_{b}
b.z^\top.z.(2{w^{[l]}_{i}}^\top(\tau-1)+b^\top)+
2 {w^{[l+1]}_{:i}}^\top.z.b^\top+2{o^{[l]}_{:i}}^\top.g^{[l+1]}.z.{b}^\top.
\end{aligned}
\end{equation}

From this equation, we determine that $b^{*}$ is given by:

\begin{equation}\begin{aligned}\label{first_equation_b*}
b^{*}=\argmax_{b}
(2w^{[l]}_{i}(\tau-1)+b).z^\top.z.b^\top+
2 {w^{[l+1]}_{:i}}^\top.z.b^\top
+2{o^{[l]}_{:i}}^\top.g^{[l+1]}.z.{b}^\top.
\end{aligned}
\end{equation}

From Lemma~\ref{g_is_in_order_alpha} we know $g^{[l+1]} \propto \alpha$ and since $a^{[l-1]}$ is not a function of $\alpha$, we can easily infer $z \propto \alpha$. In addition, we know that $\normx{b}\leq\alpha$.

Combining these two allow us to deduce that a positive constant $c_1$ exists such that $\max_{b} b.z^\top.z.b^\top < c_1 \alpha^4$. Similarly, a positive constant $c_2$ exists such that $\max_{b} 2w^{[l]}_{i}(\tau-1).z^\top.z.b^\top < c_2 \alpha^3$. Combining these two leads to the conclusion that the first addends of Eq.~\eqref{first_equation_b*} is less than $c_3 \alpha^3$ for some positive constant $c_3$.

Similar to first addends, Third addends has $g,z,b$ that are all proportional to $\alpha$. Therefore, $2{o^{[l]}_{:i}}^\top.g^{[l+1]}.z.{b}^\top < c_4 \alpha^3$ for some positive constant $c_4$.

In contrast, for the second addends of Eq.~\eqref{first_equation_b*}  there exists a positive constant $c_4$ such that, $\max_{b}2{w^{[l+1]}_{:i}}^\top.z.b^\top > c_5\alpha^2$. This is because this value scale by $\alpha^2$ when changing the $\alpha$ even if $b$ does not change in maximization for a better choice.  Therefore, we can consider only the second addends in the limit, provided that it can take non-zero values. We have: 
\begin{equation}\begin{aligned}
\max_{b}2{w^{[l+1]}_{:i}}^\top.z.b^\top=\alpha\normx{{w^{[l+1]}_{:i}}^\top.z}=\alpha\normx{{w^{[l+1]}_{:i}}^\top.{g^{[l+1]}}^\top(\tau).a^{[l-1]}(\tau)}=\\
\alpha\normx{{r^{[l]}_{:i}}^\top(\tau).a^{[l-1]}(\tau)}\neq0,
\end{aligned}\end{equation}
showing that this second addend can take non-zero values. Hence,
\begin{equation}\begin{aligned}
\label{simplifiedMax}
\lim_{\alpha\to0^+}\text{Normalized}(b^{*})=\lim_{\alpha\to0^+}\text{Normalized}(\argmax_{b}{w^{[l+1]}_{:i}}^\top.z.b^\top).
\end{aligned}\end{equation}
$b$ should aligns with the direction of ${w^{[l+1]}_{:i}}^\top.z$ to maximize the term. From this, we deduce:
\begin{equation}\begin{aligned}
\lim_{\alpha\to0^+}\text{Normalized}(b^{*})=\text{Normalized}({w^{[l+1]}_{:i}}^\top.z).
\end{aligned}\end{equation}
Substituting $z$ with its definition, we derive:
\begin{equation}\begin{aligned}
\lim_{\alpha\to0^+}\text{Normalized}(b^{*})=\text{Normalized}({w^{[l+1]}_{:i}}^\top.{g^{[l+1]}}^\top.a^{[l-1]}).
\end{aligned}\end{equation}
By employing the definition of $r^{l}$ for $l\neq L$ from Eq.~\eqref{hiddenLayer_R}, it is equal to
\begin{equation}\begin{aligned}
\text{Normalized}({r^{[l]}_{:i}}^\top.a^{[l-1]}).
\end{aligned}\end{equation}
Given that the Normalized operator nullifies the impact of scalars, it is equal to
\begin{equation}\begin{aligned}
\text{Normalized}(c^{[l]}_i{r^{[l]}_{:i}}^\top.a^{[l-1]}).
\end{aligned}\end{equation}
Upon invoking the definition of $g$ from Eq.~\eqref{genral_g}, we rewrite the equation as
\begin{equation}\begin{aligned}
\lim_{\alpha\to0^+}\text{Normalized}(b^{*})=\text{Normalized}({g^{[l]}_{:i}}^\top.a^{[l-1]}).
\end{aligned}\end{equation}
Using Lemma~\ref{changes_are_small_lemma}, we conclude that:
\begin{equation}\begin{aligned}
\lim_{\alpha\to0^+}\text{Normalized}(b^{*})=\text{Normalized}({g^{[l]}_{:i}}^\top(\tau-1).a^{[l-1]}(\tau-1)),
\end{aligned}\end{equation} which is the update performed in the \algname{}.

In order to validate the second component of the theorem for the hidden layers, let us revisit Eq.~\eqref{first_equation_b*}.

\begin{equation}\begin{aligned}
\lim_{\alpha\to0^+}\frac{\normx{b^{*}}}{\alpha}=
\lim_{\alpha\to0^+}\frac{\normx{\argmax_{b}{
(2w^{[l]}_{i}(\tau-1)+b).z^\top.z.b^\top+
2 {w^{[l+1]}_{:i}}^\top.z.b^\top
+2{o^{[l]}_{:i}}^\top.g^{[l+1]}.z.{b}^\top
}}}{\alpha}
\end{aligned}\end{equation}
\begin{equation}\begin{aligned}
=\lim_{\alpha\to0^+}\frac{\normx{\argmax_b
{w^{[l+1]}_{:i}}^\top.z.b^\top
}}{\alpha}=\lim_{\alpha\to0^+}\frac{\alpha}{\alpha}=1
\end{aligned}\end{equation}

Having demonstrated both aspects for the hidden layers, we now endeavor to establish the theorem's validity for the last layer's neurons.

Unlike hidden layers' neurons, the actions of the last layer's neurons directly impact their immediate utility. Given our assumption that neurons are greedy, they seek to optimize their utility in iteration $\tau$. In addition, all neurons in the last layer share the objective function, creating a fully collaborative setting. Therefore, we shall consider all neurons within the final layer instead of examining a solitary neuron. Let us define $b=w^{[L]}(\tau)-w^{[L]}(\tau-1)$; this redefinition transforms the limitation on $b$ to $\forall_i: \normx{b_i}\leq\alpha$.  
Additionally, as the loss is a function of $a^{[L]}$, we can perceive it as a function of $a^{[L-1]}$ and $w^{[L]}$. Using this, we obtain,
\begin{equation}\begin{aligned}
\max_b U^{[L]}(\tau)=\max_b -\ell(a^{[L-1]},{w^{[L]}}(\tau))=-\min_b \ell(a^{[L-1]},{w^{[L]}}(\tau-1)+b)
\end{aligned}\end{equation}
From the update rule of the \algname{} for the last layer (Eq.~\eqref{lastLayer_R}), we have that the loss is differentiable with respect to $w^{[L]}(\tau-1)$. As we are interested in examining the behavior when $\alpha$ approaches zero, we can consider the loss function as a linear function within that neighborhood and compute it using a first-order Taylor, which is precise in our case. So
\begin{equation}\begin{aligned}
\label{last_layer_matrix_frobenius}
\max_b U^{[L]}(\tau)=\max_b -\ell(a^{[L-1]},{w^{[L]}}(\tau-1))-\langle\frac{\partial \ell}{\partial w^{[L]}}({w^{[L]}}(\tau-1)),b \rangle_F
\end{aligned}\end{equation}
The first term does not depend on $b$, and using the definition of $r$ from Eq.~\eqref{lastLayer_R}, it equates to:
\begin{equation}\begin{aligned}
-\ell(a^{[L-1]},{w^{[L]}}(\tau-1))+ \max_b \langle {r^{[L]}}^\top(\tau-1).a^{[L-1]},b \rangle_F\\
=-\ell(a^{[L-1]},{w^{[L]}}(\tau-1))+ \max_b \sum_i \langle {r^{[L]}_{:i}}^\top(\tau-1).a^{[L-1]},b \rangle_F\\
=-\ell(a^{[L-1]},{w^{[L]}}(\tau-1))+ \sum_i \max_{b_i}  {r^{[L]}_{:i}}^\top(\tau-1).a^{[L-1]}.b_i^\top.
\end{aligned}\end{equation}
Therefore, the direction of $b^*_i$ is the same as ${r^{[L]}_{:i}}^\top(\tau-1).a^{[L-1]}$, and its length should be maximized ($\alpha$), which completes the proof for the last layer's neurons.
\end{proof}

\textbf{Theorem 4.}
Consider an MLP with neurons using a piecewise linear activation functions. Assume that $\forall_{l,i,s}: a^{[l]}_{s,i}\neq0$, and for all $i$ and $l$ the following conditions hold:\\
1) $0 < \normx{{r^{[l]}_{:i}}^\top.a^{[l-1]}}$,\\
2) $g^{[l+1]}(\tau)$ is independent of neuron $i$'s response at iteration $\tau-1$.\\
Then, the responses of all neurons converge to the \algname{} response as follows:

1) $\lim_{\alpha \to 0^+} \text{normalize}({{b^{[l]}_i}^*(\tau-1)}) = \text{normalize}({b^{[l]}_i(\tau-1)})$,\\
2) $\lim_{\alpha \to 0^+} \frac{\normx{{b^{[l]}_i}^*(\tau-1)}}{\normx{b^{[l]}_i(\tau-1)}} = 1$,\\
for every neuron $i$ in layer $l$, and every layer $l$ in the MLP, where \textit{normalize} refers to $\ell_2$ vector normalization operator. 

\begin{proof}[Proof of Theorem~\ref{convergence_theorem}]
The proof is carried out through induction. For the base case, applying Theorem~\ref{semi_is_nash_label}, we affirm that all responses of the last layer's neurons converge to \algname{}, satisfying both parts of the theorem's conclusion. 

In the induction step, we aim to demonstrate that all neurons in the $l$-th layer have their best responses converging to \algname{}, given that all subsequent layer neurons' behavior converges to \algname{}. For this, we again invoke Theorem~\ref{semi_is_nash_label}, under which all conditions are satisfied. Since the theorem examines the behavior of neurons as $\alpha$ approaches zero and subsequent layers neurons' behavior converge to \algname{} under that condition, the induction step is validated, leading us to the conclusion of the theorem.
\end{proof}

\subsubsection{Proofs of Theorem~\ref{semi_is_nash_label_CNN} and Theorem~\ref{convergence_theorem:CNN}}
Before proving Theorem~\ref{semi_is_nash_label_CNN}, we need four important lemmas.

\begin{lemma} 
\label{f_dot_rconv_to_deconv}
Given matrices $A\in R^{n\times 1 \times k}$, $B\in R^{m\times n \times d}$, $C\in R^{m\times 1 \times d+k-1}$, then:
\begin{equation}\begin{aligned}
\fdot{A}{\rconv{B}{C}}=\fdot{\deconv{B}{A}}{C}
\end{aligned}\end{equation}
\end{lemma}
\begin{proof}[Proof of Lemma~\ref{f_dot_rconv_to_deconv}]
The definition of the $\psi$ function and the Frobenius inner product gives us:
\begin{equation}\begin{aligned}
 \fdot{A}{\rconv{B}{C}}
 =\sum_{i,e}^{n,k}A_{i,1,e}\rconv{B}{C}_{i,1,e}
 =\sum_{i,e}^{n,k}A_{i,1,e}\sum_{s,p}^{m,d} B_{s,i,p}C_{s,1,p-e+k}.
  \end{aligned}\end{equation}
defining variable $x=p-e+k$, we have
\begin{equation}\begin{aligned}
 \fdot{A}{\rconv{B}{C}}=\sum_{i,e}^{n,k}\sum_{s,x=1-e+k}^{m,d-e+k} A_{i,1,e}B_{s,i,x+e-k}C_{s,1,x}
  \end{aligned}\end{equation}
  \begin{equation}\begin{aligned}
 =\sum_{s,x=1}^{m,d+k-1} C_{s,1,x} \sum_{i,e=k+1-x}^{n,d+k-x} A_{i,1,e}B_{s,i,x+e-k}
  \end{aligned}\end{equation}
\begin{equation}\begin{aligned}
 =\sum_{s,x=1}^{m,d+k-1} C_{s,1,x} \deconv{B}{A}_{s,1,x}
 =\fdot{C}{\deconv{B}{A}}.
  \end{aligned}\end{equation}
\end{proof}

\begin{lemma} 
\label{f_dot_conv_to_rconv2}
Given matrices $A\in R^{n\times 1 \times k}$, $B\in R^{n\times p \times 2k-1}$, $C\in R^{1\times p \times k}$, then:
\begin{equation}\begin{aligned}
\fdot{A}{\conv{B}{C}}=\fdot{\rconv{A}{B}}{C}
\end{aligned}\end{equation}
\end{lemma}
\begin{proof}[Proof of Lemma~\ref{f_dot_conv_to_rconv2}]
The definition of the $Conv$ function and the Frobenius inner product give us:
\begin{equation}\begin{aligned}
 \fdot{A}{\conv{B}{C}}
 =\sum_{i,x}A_{i,1,x}\conv{B}{C}_{i,1,x}
 =\sum_{i,x}A_{i,1,x} \sum_{y=1}^{p}\sum_{e=1}^{k} B_{i,y,x-e+k}.C_{1,y,e}.
  \end{aligned}\end{equation}
  By changing the order of summations and extracting $C$ from the inner summation, this value can be rewritten as:
 \begin{equation}\begin{aligned}
  \sum_{y=1, e=1}^{p,k} \sum_{i=1,x=1}^{n,k} A_{i,1,x}.B_{i,y,x-e+k}.C_{1,y,e}
  =\sum_{y=1, e=1}^{p,k} C_{1,y,e} \sum_{i=1,x=1}^{n,k} A_{i,1,x}.B_{i,y,x-e+k}.
   \end{aligned}\end{equation}
   Finally, applying the definition of $\psi$ and the Frobenius inner product once more, it is equal to:
 \begin{equation}\begin{aligned}
  \sum_{y=1, e=1}^{p,k} C_{1,y,e} \rconv{A}{B}_{1,y,e}
  =\fdot{\rconv{A}{B}}{C}.
 \end{aligned}\end{equation}
\end{proof}

\begin{lemma} 
\label{g_is_in_order_alpha:CNN}
Having a CNN with filters using a piecewise linear function activation in \algname{} algorithm, $\forall_l: g^{[l]} \propto \alpha$, and $\forall_{l\neq L}: r^{[l]} \propto \alpha$.
\end{lemma}
\begin{proof}[Proof of Lemma~\ref{g_is_in_order_alpha:CNN}]
we use induction to prove this Lemma. Starting from the last layer, according to Eq.~\eqref{lastLayer_R:CNN}, we know that $r^{[L]}$ is independent of $\alpha$. Given that from Eq.~\eqref{general_c:CNN} we conclude that $c^{[L]}\propto \alpha$,  applying Eq.~\eqref{genral_g:CNN} leads to the conclusion that $g^{[L]} \propto \alpha$.

Now, we aim to demonstrate the induction step. By assuming $g^{[l+1]} \propto \alpha$, we need to prove that $g^{[l]} \propto \alpha$ and $r^{[l]} \propto \alpha$.  Considering Eq.~\eqref{hiddenLayer_R:CNN}, and since $w^{[l+1]}$ is not a function of $\alpha$, we can assert that $r^{[l]} \propto \alpha$. By the definition of $\propto$, there exists ${r^{[l]}}^*$ such that $r^{[l]}(\alpha) = \alpha.{r^{[l]}}^*$.
Now from Eq.~\eqref{general_c:CNN}, we have, 
\begin{equation}\begin{aligned}
        c^{[l]}_i = \frac{\alpha^{[l]}_i}{\sqrt{\sum_{j,e} ({\sum_{s,p} r^{[l]}_{s,i,p}.a^{[l-1]}_{s,j,p-e}})^2 }}
        =\frac{\alpha^{[l]}_i}{\sqrt{\sum_{j,e} ({\sum_{s,p} {\alpha r^{[l]}_{s,i,p}}^*.a^{[l-1]}_{s,j,p-e}})^2 }}
\end{aligned}\end{equation}
\begin{equation}\begin{aligned}
        =\frac{1}{\sqrt{\sum_{j,e} ({\sum_{s,p} {r^{[l]}_{s,i,p}}^*.a^{[l-1]}_{s,j,p-e}})^2 }}.
\end{aligned}\end{equation}
Thus, it is shown that $c^{[l]}$ is not a function of $\alpha$.  Therefore
\begin{equation}
    g^{[l]}_i=r^{[l]}_i.c^{[l]}_i\propto \alpha,
\end{equation}
which proves the induction step and completes the proof.
\end{proof}

\begin{lemma} 
\label{changes_are_small_lemma:CNN}
Suppose we have a CNN with filters using a piecewise linear function activation and \\$\forall_{l,j,s,d}: a^{[l]}_{s,j,d}\neq0$. Then we have:
\begin{equation}\begin{aligned}
\forall_{l,i}: \lim_{\alpha\to0^+}:\rconv{g^{[l]}_{:i:}}{a^{[l-1]}}=\rconv{g^{[l]}_{:i:}(\tau-1)}{a^{[l-1]}(\tau-1)},
\end{aligned}\end{equation}
and 
\begin{equation}\begin{aligned}
\forall_{l}: \lim_{\alpha\to0^+}:M^{[l]}(\tau)=M^{[l]}(\tau-1).
\end{aligned}\end{equation}
\end{lemma}
\begin{proof}[Proof of Lemma~\ref{changes_are_small_lemma:CNN}]
We first demonstrate that $\forall_{l}: \lim_{\alpha\to0^+}: \mathds{I}(a^{[l]}(\tau)>0)=\mathds{I}(a^{[l]}(\tau-1)>0)$.
This is because  $\forall_{l,j,s}: a^{[l]}_{s,j}\neq0$. If we have an $\alpha<\min_{l,j,s} | a^{[l]}_{s,j}|$. Then with such $\alpha$, none of the neurons switch from active to inactive or vice versa.

In addition, when $\alpha$ approaches zero $\forall_l: w^{[l]}(\tau)$ converges to $w^{[l]}(\tau-1)$, and considering that the dataset remains unchanged during iterations, we establish that $\forall_{l}: \lim_{\alpha\to0^+}:a^{[l]}=a^{[l]}(\tau-1)$. 
We can then deduce the same statement for $g$; namely, $\forall_{l}: \lim_{\alpha\to0^+}:g^{[l]}=g^{[l]}(\tau-1)$. 
Combining these results validates the lemma.
\end{proof}

\textbf{Theorem 3.}
Consider a CNN with a piecewise linear activation functions. Where all $\alpha^{[l]}_j$ are equal to $\alpha$, and all filters in layers after the $i$-th layer employ \algname{}, and $\forall_{l,j,s,d}: a^{[l]}_{s,j,d}\neq 0$ and:\\
1) $\exists c,j: \sum_{s=1}^{m}{\sum_{p=1}^{D}r^{[l]}_{s,i,p}a^{[l-1]}_{s,c,p-j}} \neq 0$,\\
2) $g^{[l+1]}(\tau)$ is independent of filter $i$'s response at iteration $\tau-1$,\\ then:\\
1)~$
\lim_{\alpha \to 0^+} \text{normalize}({{b^{[l]}_i}^*(\tau-1)}) = \text{normalize}({b^{[l]}_i(\tau-1)})
$,\\
2)~$
\lim_{\alpha \to 0^+} \frac{\normx{{b^{[l]}_i}^*(\tau-1)}}{\normx{b^{[l]}_i(\tau-1)}} = 1$.\\
Here,  \textit{normalize} signifies an operation that scales the matrix such that its Frobenius norm becomes equivalent to 1.

\begin{proof}[Proof of Theorem~\ref{semi_is_nash_label_CNN}]
First, we prove the theorem for hidden layers' filters. We want to prove that filter $i$ in layer $l\neq L$ best strategy is converging to \algname{} under the Theorem's assumptions. Suppose this filter modifies its input weights with vector $b^{*} \in \mathds{R}^{1 \times n_{l-1} \times k_l}$, thus  $w^{[l]}_i=w^{[l]}_i(\tau-1)+b^{*}$. The action of this filter impacts its utility function in iteration $\tau+1$. Given our assumption of rational and greedy filters, it will attempt to optimize its utility function in iteration $\tau+1$. By Eq.~\eqref{utility_equation:CNN} we have,

\begin{equation}
U^{[l]}_i(\tau+1) =\normf{w^{[l+1]}_{:i:}(\tau+1)}.
\end{equation}
Considering that subsequent filters employ \algname{} for updating their input weight, from Eq.~\eqref{eq:update_rule:CNNshort} the value is equivalent to:
\begin{equation}
\normf{{w^{[l+1]}}_{:i:}+\rconv{g^{[l+1]}}{a^{[l]}_{:i:}}}.
\end{equation}
Following the neuron rationality and greediness, we drive:
\begin{equation}
U^{[l+1]}_i(\tau+1) = \max_{b} \normf{w^{[l+1]}_{:i:}+\rconv{g^{[l+1]}}{a^{[l]}_{:i:}}},
\end{equation}
where $\normf{b}\leq\alpha$. Given the strictly non-negative values of $U$, we can infer that the optimization problem of maximizing $U$ corresponds to maximizing $U^2$. So we investigate the latter, 
\begin{equation}
(U^{[l+1]}_i(\tau+1))^2 = \max_{b} \normf{w^{[l+1]}_{:i:}+\rconv{g^{[l+1]}}{a^{[l]}_{:i:}}}^2
\end{equation}
\begin{equation}
=\max_{b} \normf{w^{[l+1]}_{:i:}}^2+\normf{\rconv{g^{[l+1]}}{a^{[l]}_{:i:})}}^2+2 \fdot{w^{[l+1]}_{:i}}{\rconv{g^{[l+1]}}{a^{[l]}_{:i:})}}.
\end{equation}
Using Lemma~\ref{f_dot_rconv_to_deconv}, this is equal to
\begin{equation}
 \max_{b} \normf{w^{[l+1]}_{:i:}}^2 + \normf{\rconv{g^{[l+1]}}{a^{[l]}_{:i:}}}^2+2 \fdot{\deconv{g^{[l+1]}}{w^{[l+1]}_{:i:}}}{a^{[l]}_{:i:}}.
\end{equation}
Using forward pass definition in Eq.~\eqref{forward_pass:CNN:short} we can rewrite this as:  
\begin{equation}\begin{aligned}
\label{first_equation_using_forward_pass:CNN}
\max_{b} \normf{w^{[l+1]}_{:i:}}^2 +\normf{\rconv{g^{[l+1]}}{\conv{a^{[l-1]}}{w^{[l]}_{i}}\odot M^{[l]}_{:i:}+o^{[l]}_{:i:} }}^2\\
+2 \fdot{\deconv{g^{[l+1]}}{w^{[l+1]}_{:i:}}}{\conv{a^{[l-1]}}{w^{[l]}_{i}}\odot M^{[l]}_{:i:}+o^{[l]}_{:i:}},
\end{aligned}\end{equation}

where $M^{[l]}={f^{[l]}}^{'}(a^{[l]})$ and $o^{[l]}$ is the offset of the linear function $f^{[l]}$ in $a^{[l]}$. Using $w^{[l]}_i=w^{[l]}_i(\tau-1)+b$, we deduce:
\begin{equation}\begin{aligned}
(U^{[l+1]}_i(\tau+1))^2 = \max_{b} \normf{w^{[l+1]}_{:i:}}^2 +\normf{\rconv{g^{[l+1]}}{\conv{a^{[l-1]}}{w^{[l]}_{i}(\tau-1)+b}\odot M^{[l]}_{:i:} +o^{[l]}_{:i:} }}^2\\
+2 \fdot{\deconv{g^{[l+1]}}{w^{[l+1]}_{:i:}}}{\conv{a^{[l-1]}}{w^{[l]}_{i}(\tau-1)+b}\odot M^{[l]}_{:i:} +o^{[l]}_{:i:} }
\end{aligned}\end{equation}
\begin{equation}\begin{aligned}
(U^{[l+1]}_i(\tau+1))^2 = \max_{b} \normf{w^{[l+1]}_{:i:}}^2 +\normf{\rconv{g^{[l+1]}}{\conv{a^{[l-1]}}{w^{[l]}_{i}(\tau-1)+b}\odot M^{[l]}_{:i:} }}^2\\
+2 \fdot{\deconv{g^{[l+1]}}{w^{[l+1]}_{:i:}}}{\conv{a^{[l-1]}}{w^{[l]}_{i}(\tau-1)+b}\odot M^{[l]}_{:i:} }\\
+\normf{\rconv{g^{[l+1]}}{\conv{a^{[l-1]}}{o^{[l]}_{:i:}} }}^2
+2\fdot{\deconv{g^{[l+1]}}{w^{[l+1]}_{:i:}}}{\conv{a^{[l-1]}}{o^{[l]}_{:i:}} }
\end{aligned}\end{equation}

\begin{equation}\begin{aligned}
=\max_{b} 
\normf{w^{[l+1]}_{:i:}}^2 +\normf{\rconv{g^{[l+1]}}{\conv{a^{[l-1]}}{w^{[l]}_{i}(\tau-1)}\odot M^{[l]}_{:i:}  }}^2\\
+\normf{\rconv{g^{[l+1]}}{\conv{a^{[l-1]}}{b}\odot M^{[l]}_{:i:} }}^2\\
+\fdot{\rconv{g^{[l+1]}}{\conv{a^{[l-1]}}{w^{[l]}_{i}(\tau-1)}\odot M^{[l]}_{:i:} }}{\rconv{g^{[l+1]}}{\conv{a^{[l-1]}}{b}\odot M^{[l]}_{:i:} }}\\
+2 \fdot{\deconv{g^{[l+1]}}{w^{[l+1]}_{:i:}}}{\conv{a^{[l-1]}}{w^{[l]}_{i}(\tau-1)}\odot M^{[l]}_{:i:}}\\
+2 \fdot{\deconv{g^{[l+1]}}{w^{[l+1]}_{:i:}}}{\conv{a^{[l-1]}}{b}\odot M^{[l]}_{:i:}}\\
+\normf{\rconv{g^{[l+1]}}{\conv{a^{[l-1]}}{o^{[l]}_{:i:}} }}^2
+2\fdot{\deconv{g^{[l+1]}}{w^{[l+1]}_{:i:}}}{\conv{a^{[l-1]}}{o^{[l]}_{:i:}}}
\end{aligned}\end{equation}
Given the theory assumptions, we know $g^{[l+1]}(\tau)$ is independent of the choice of $b$. Also, By Lemma \ref{changes_are_small_lemma:CNN} we have $M^{[l]}(\tau)$ and $o^{[l]}$ are independent of the choice of $b$. In addition $w^{[l+1]}(\tau)$and $a^{[l-1]}(\tau)$ are independent of the choice of $b$ by their definition. Using this statement, we can decouple some of the terms from the maximum and drive :
\begin{equation}\begin{aligned}
\label{temporal:equation:92}
b^*=\argmax_{b}
\normf{\rconv{g^{[l+1]}}{\conv{a^{[l-1]}}{b}\odot M^{[l]}_{:i:} }}^2\\
+\fdot{\rconv{g^{[l+1]}}{\conv{a^{[l-1]}}{w^{[l]}_{i}(\tau-1)}\odot M^{[l]}_{:i:} }}{\rconv{g^{[l+1]}}{\conv{a^{[l-1]}}{b}\odot M^{[l]}_{:i:} }}\\
+2 \fdot{\deconv{g^{[l+1]}}{w^{[l+1]}_{:i}}}{\conv{a^{[l-1]}}{b}\odot M^{[l]}_{:i:}}.
\end{aligned}\end{equation}
Since $\fdot{.}{.}$ calculate a sum over multiplication of elements, we can transfer $M^{[l]}$ to the left side of the last term; therefore:
\begin{equation}\begin{aligned}
\label{temporal:equation:91}
b^*=\argmax_{b}
\normf{\rconv{g^{[l+1]}}{\conv{a^{[l-1]}}{b}\odot M^{[l]}_{:i:} }}^2\\
+\fdot{\rconv{g^{[l+1]}}{\conv{a^{[l-1]}}{w^{[l]}_{i}(\tau-1)}\odot M^{[l]}_{:i:} }}{\rconv{g^{[l+1]}}{\conv{a^{[l-1]}}{b}\odot M^{[l]}_{:i:} }}\\
+2 \fdot{\deconv{g^{[l+1]}}{w^{[l+1]}_{:i}}\odot M^{[l]}_{:i:}}{\conv{a^{[l-1]}}{b}}.
\end{aligned}\end{equation}
Using definition of $r$ from Eq.~\eqref{hiddenLayer_R:CNN}, we obtain
\begin{equation}\begin{aligned}\label{first_equation_b*:CNN}
b^*=\argmax_{b}
\normf{\rconv{g^{[l+1]}}{\conv{a^{[l-1]}}{b}\odot M^{[l]}_{:i:} }}^2\\
+\fdot{\rconv{g^{[l+1]}}{\conv{a^{[l-1]}}{w^{[l]}_{i}(\tau-1)}\odot M^{[l]}_{:i:} }}{\rconv{g^{[l+1]}}{\conv{a^{[l-1]}}{b}\odot M^{[l]}_{:i:} }}\\
+2 \fdot{r^{[l]}_{:i:}}{\conv{a^{[l-1]}}{b}}.
\end{aligned}\end{equation}
 From Lemma~\ref{g_is_in_order_alpha:CNN} we know $g^{[l+1]} \propto \alpha$, and $r^{[l]} \propto \alpha$. In addition, we have $\normf{b}\leq\alpha$.

Combining these two allow us to deduce that a positive constant $c_1$ exists such that $\max_{b} \normf{\rconv{g^{[l+1]}}{\conv{a^{[l-1]}}{b}\odot M^{[l]}_{:i:} }}^2 < c_1 \alpha^4$. Similarly, a positive constant $c_2$ exists such that 
\begin{equation}\begin{aligned}
\max_{b}\fdot{\rconv{g^{[l+1]}}{\conv{a^{[l-1]}}{w^{[l]}_{i}(\tau-1)}\odot M^{[l]}_{:i:} }}{\rconv{g^{[l+1]}}{\conv{a^{[l-1]}}{b}\odot M^{[l]}_{:i:} }} < c_2 \alpha^3.
\end{aligned}\end{equation}

In contrast, for the third addends of Eq.~\eqref{first_equation_b*:CNN},  there exists a positive constant $c_4$ such that, $\max_{b}2 \fdot{r^{[l]}_{:i:}}{\conv{a^{[l-1]}}{b}} > c_4\alpha^2$. This is because this value scale by $\alpha^2$ when changing the $\alpha$ even if $b$ does not change in maximization for a better choice. Therefore, we can consider only the third addend in the limit, provided that it can take non-zero values. It is the case because,

\begin{equation}\begin{aligned}\label{Temporal:Why third is none negetive}
\max_{b} 2 \fdot{r^{[l]}_{:i:}}{\conv{a^{[l-1]}}{b}}=
\max_{b} 2 \fdot{\rconv{r^{[l]}_{:i:}}{a^{[l-1]}}}{b}>0.
\end{aligned}\end{equation}
The equality in \eqref{Temporal:Why third is none negetive} is obtained using Lemma~\ref{f_dot_conv_to_rconv2}, and the inequality is 
due to the assumption (1) of the Theorem, which states that there exists a non-zero cell in $\rconv{r^{[l]}_{:i:}}{a^{[l-1]}}$.
Therefore, it is viable for the third addend to take non-zero values, and we can consider only the third addend.
Hence, 
\begin{equation}\begin{aligned}
\lim_{\alpha\to0^+}\normalizeT{b^{*}}=\lim_{\alpha\to0^+}\normalizeT{\argmax_{b}
\fdot{r^{[l]}_{:i:}}{\conv{a^{[l-1]}}{b}}}.
\end{aligned}\end{equation}
Given that a scalar does not have any effect on Argmax, it is equal to
\begin{equation}\begin{aligned}
\lim_{\alpha\to0^+}\normalizeT{\argmax_{b}
\fdot{c^{[l]}_i r^{[l]}_{:i:}}{\conv{a^{[l-1]}}{b}}}.
\end{aligned}\end{equation}
Using definition of $g$ from Eq.~\eqref{general_c:CNN}, it is equal to
\begin{equation}\begin{aligned}
\lim_{\alpha\to0^+}\normalizeT{\argmax_{b}
\fdot{g^{[l]}_{:i:}}{\conv{a^{[l-1]}}{b}}}.
\end{aligned}\end{equation}
Using Lemma~\ref{f_dot_conv_to_rconv2}, it is equal to:
\begin{equation}\begin{aligned}
\lim_{\alpha\to0^+}\normalizeT{\argmax_{b}
\fdot{\rconv{g^{[l]}_{:i:}}{a^{l-1]}}}{b}}=\lim_{\alpha\to0^+}\normalizeT{\rconv{g^{[l]}_{:i:}}{a^{l-1]}}}.
\end{aligned}\end{equation}
Using Lemma~\ref{changes_are_small_lemma:CNN} we drive:
\begin{equation}\begin{aligned}
\lim_{\alpha\to0^+}\normalizeT{b^*}=\lim_{\alpha\to0^+}\normalizeT{\rconv{g^{[l]}_{:i:}(\tau-1)}{a^{l-1]}(\tau-1)}}.
\end{aligned}\end{equation}
which is the update performed in the \algname{}.

To prove the second part of the Theorem for the hidden layers, we consider the transition from Eq.~\eqref{first_equation_b*:CNN} and the subsequent modifications imposed. Accordingly, we arrive at the following:
\begin{equation}\begin{aligned}
\lim_{\alpha\to0^+}\frac{\normf{b^{*}}}{\alpha}=
\lim_{\alpha\to0^+}\frac{\normf{\argmax_{b}
\fdot{\rconv{g^{[l]}_{:i:}}{a^{l-1]}}}{b}
}}{\alpha}=\lim_{\alpha\to0^+}\frac{\alpha}{\alpha}=1
\end{aligned}\end{equation}

The above equation allows us to affirm both aspects of the Theorem about the hidden layers. Given that the proof for the last layer filters mirrors that of the last layer's neurons as delineated in Theorem~\ref{semi_is_nash_label}, we omit this part of the proof.
\end{proof}

Following the establishment of Theorem~\ref{semi_is_nash_label_CNN}, now we can introduce the parallel counterpart of Theorem~\ref{convergence_theorem} in the CNN domain, presented as Theorem~\ref{convergence_theorem:CNN}.

\begin{theorem}
\label{convergence_theorem:CNN}
Consider a CNN with a piecewise linear activation function. Assume for all $i$ and $l$ the following conditions hold:\\
1) $\exists c,j: \sum_{s=1}^{m}{\sum_{p=1}^{D}r^{[l]}_{s,i,p}a^{[l-1]}_{s,c,p-j}} \neq 0$,\\
2) $g^{[l+1]}(\tau)$ is independent of filter $i$'s response at iteration $\tau-1$,\\ then:\\
3) $\forall_{s, j}: z^{[l]}_{s,i,j}\notin f^{[l]} \text{breakpoints}$

1)~$
\lim_{\alpha \to 0^+} \text{normalize}({{b^{[l]}_i}^*(\tau-1)}) = \text{normalize}({b^{[l]}_i(\tau-1)})
$,\\
2)~$
\lim_{\alpha \to 0^+} \frac{\normx{{b^{[l]}_i}^*(\tau-1)}}{\normx{b^{[l]}_i(\tau-1)}} = 1$.\\
for every filter $i$ in layer $l$, and every layer $l$ in the CNN.
\end{theorem}
\begin{proof}[Proof of Theorem~\ref{convergence_theorem:CNN}]
The argument validating this Theorem remains consistent with the proof supporting Theorem~\ref{convergence_theorem}. 
Assuming that the response of the subsequent layer converges to \algname{}, when this assumption is combined with existing premises, it allows us to affirm that the response of the present layer also converges to \algname{}, using Theorem \ref{semi_is_nash_label_CNN}. Such a proposition enables the proof of the Theorem using induction, beginning with the final layer. Also, the validity of the Theorem for the final layer remains independent of the responses of other neurons; therefore, it can serve as the induction base.
\end{proof}

\section{Experiments Detail}
\label{Experiments:Appendix}
\subsection{Datasets}
 \textbf{MNIST} is a dataset of 60,000 grayscale handwritten digit images, with each image belonging to one of ten classes. Similar to CIFAR-10, 10,000 images are reserved for testing purposes, and the remaining 50,000 images are used for training. Each image in this dataset has a resolution of $28\times28$ pixels.

\textbf{Split-MNIST} is a variation of the original MNIST dataset, wherein the ten classes are divided into five distinct groups, each containing a random pair of classes. In this setting, a model is trained to perform binary classification tasks by distinguishing between the two classes within a group. Following a predetermined number of training epochs, the model proceeds to the next group and is trained on the binary classification task associated with that group. This process is repeated until the model has been trained on all five groups.

\textbf{CIFAR-10}  is a dataset comprising 60,000 natural images, each assigned to one of ten distinct classes. A subset of 10,000 images is designated as the test dataset, while the remaining 50,000 images constitute the training dataset. Each image in the dataset has a resolution of $32\times32$ pixels and consists of three color channels.

\subsection{Data Preprocessing and Validation}
In the data preprocessing stage, different techniques are applied to the training, validation, and testing datasets. For the MNIST and Split-MNIST training dataset, a random crop method is employed as a data augmentation technique, and for CIFAR-10, we use random horizontal flip, random rotation, random affine, and color jittering to increase the network's ability to generalize to unseen data. However, these methods are not applied to the validation and testing datasets to ensure a fair evaluation of the model's performance. Images are normalized prior to being fed into the model to help training. 

Before initiating the training and hyperparameter grid search, a validation set is constructed using a portion of the training samples. The validation dataset consists of 10,000 images, leaving 40,000 images for training purposes. For the Split-MNIST dataset, hyperparameter tuning is conducted on an entirely distinct class division to ensure an unbiased evaluation.

\subsection{Models' Architecture}
We employ different network architectures depending on the dataset being used. For MNIST dataset, a Multi-Layer Perceptron (MLP) is utilized. The MLP is designed with four hidden layers, each containing 2,000 neurons and ReLU (Rectified Linear Unit) activation functions. The output layer uses Softmax to produce probabilities. For the Split-MNIST experiment, we employ the same network architecture as previously described; however, the output layer is modified to consist of only two neurons, corresponding to the binary classification tasks. As the model progresses through each task, the same neurons are utilized for classification in the subsequent tasks. 

On the other hand, for the CIFAR-10 dataset, we employ the LeNet architecture \cite{lecun1998gradient}, which is a convolutional neural network (CNN) consisting of two sets of convolutional and max pooling layers, followed by two fully connected layers and an output layer which uses Softmax to produce probabilities. 
The first convolutional layer has a kernel size of $5\times5$, a stride of $1$, and six output feature maps, while the second convolutional layer has a kernel size of $5\times5$, a stride of $1$, and sixteen output feature maps. Both max pooling layers have a kernel size of $2\times2$ and a stride of $2$. The first fully connected layer has $120$ neurons, the second has $84$ neurons, and the output layer has ten neurons corresponding to the ten classes in the dataset. ReLU activation functions are used between the linear layers as a non-linearity.

\subsection{Model Training}
In the experimental phase of this study, Stochastic Gradient Descent (SGD) was employed as Error Backkpropagation baseline. To optimize the learning process, a learning rate scheduler was implemented, designed to decrease the learning rate by a factor of 0.1 in the middle of the training process. For the subset of experiments focused on catastrophic forgetting, the same learning rate scheduler was utilized for each individual task. 

\section{Additional Studies}
In addition to the experimental results in the main body of the paper we present additional results in this section. In~\ref{appendix:depthanalysis} we extend our depth analysis to a new architecture and dataset. 

\subsection{Depth analysis}
The findings illustrated in Fig.~\ref{fig:ResNet50_depth_test} demonstrate that \algname{} can be scaled to deep models, butwith a reduction in accuracy when compared to Error Backpropagation. Given that this architecture is designed for Error Backpropagation, the suboptimal performance exhibited by \algname{} is to be expected.

\label{appendix:depthanalysis}
\begin{figure}
    \centering
    \includegraphics[width=0.6\textwidth]{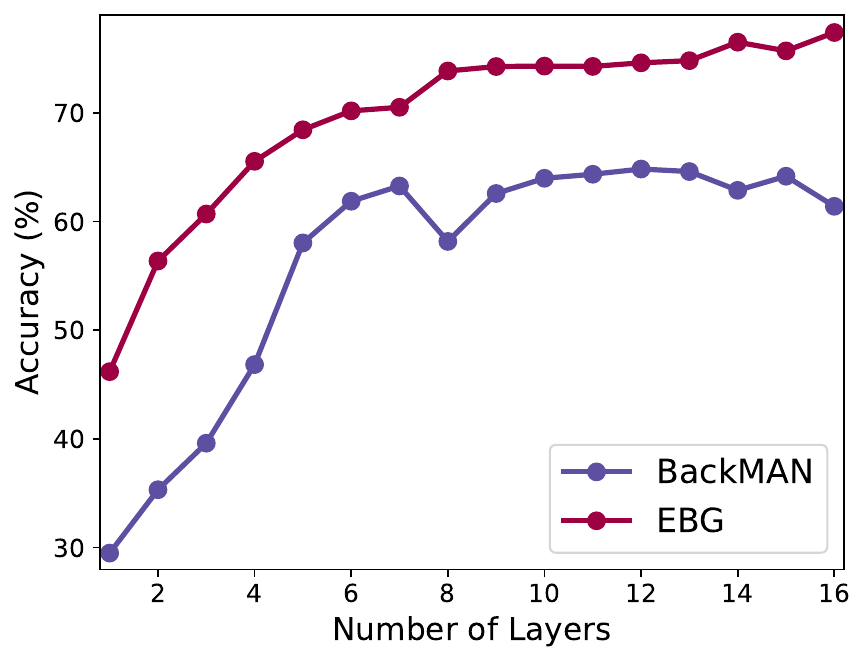}
      \caption{
The analysis examines the impact of network depth on the performance of Error Backpropagation and \algname{} algorithms, employing the ResNet-50 architecture and the CIFAR-100 dataset. ResNet-50 comprises 16 blocks; the accuracy is evaluated when a suffix of these blocks is removed, effectively varying the network depth. For each specified depth, models are trained from scratch.}
    \label{fig:ResNet50_depth_test}
\end{figure}

\subsection{Robustness to Hyperparameters}
\label{appendix:robustnesstohyperparameters}
One crucial aspect of a training algorithm is the ease of hyperparameter tuning associated with it. In Table.~\ref{Table:hyperparametertuning}, we present the accuracy of the MNIST experiment discussed in Section~\ref{experiment_classification} , but with all hyperparameters as opposed to only the fine-tuned ones. The results suggest that tuning \algname{} is not significantly more challenging than tuning Error Backpropagation.

\begin{table}
    \label{Table:hyperparametertuning}
  \caption{Results of the hyperparameter tuning for the MNIST experiment. The model is composed of four hidden linear layers with ReLU activation in between. The training has 25 epochs. The values presented indicate the accuracy of the model.}
  \centering
  \begin{tabular}{ccccccccccc}
    \toprule
     & \multicolumn{5}{c}{Error-Backpropagation} & \multicolumn{5}{c}{\algname{}} \\
     \cmidrule(lr){2-6}  \cmidrule(l){7-11}
    Learning Rate/Batch Size     & 64   & 128  & 256  & 512  & 1024 & 64   & 128  & 256  & 512  & 1024  \\
    \midrule
     0.002  & 98.5 & 98.0 & 96.6 & 93.8 & 90.0 & 93.9 & 90.1 & 75.5 & 57.8 & 52.0 \\
     0.01   & 98.8 & 98.9 & 98.6 & 98.3 & 97.1 & 98.5 & 97.1 & 95.2 & 91.2 & 78.2 \\
     0.05   & 99.2 & 99.2 & 98.9 & 98.9 & 98.6 & 99.1 & 98.8 & 98.1 & 97.4 & 96.4 \\
     0.25   & 99.2 & 99.3 & 99.1 & 99.1 & 98.4 & 0.1  & 98.3 & 98.7 & 98.7 & 98.4 \\
    \bottomrule
  \end{tabular}
\end{table}

\subsection{Damping Effect of \algname{} on Signals}
\label{appendix:DampingEffect}
\algname{} modifies gradient-like signals by multiplying them with a scalar $c$ before backpropagation. To understand the role of this variable during training, we conduct an experiment and visualize the $c$ value in Fig.~\ref{fig:DampingEffect}. The results show when norm of the signal is higer the $c$ value is lower; therefore, it reduces the norm of the backpropagated signal. This results could explain the results of Fig.~\ref{fig:depth}. For very deep neural network when gradient wanishing and gradint explotion is a serious problem, we see that \algname{} surpass error-backpropagation in terms of accuracy which could show that the model suffer less from these two problems. 

\begin{figure}
    \centering
     \begin{adjustbox}{max width=1.2\textwidth,center}
    \includegraphics{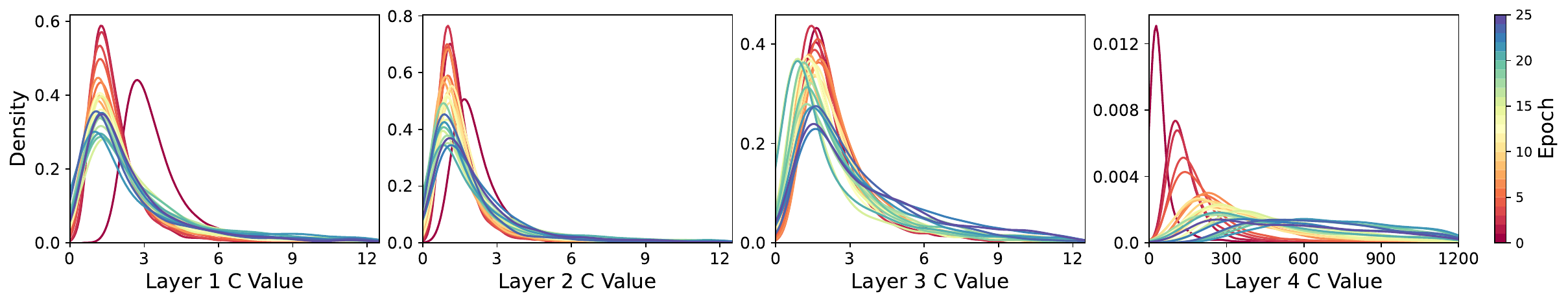}
    \end{adjustbox}
    \begin{adjustbox}{max width=1.2\textwidth,center}
    \includegraphics{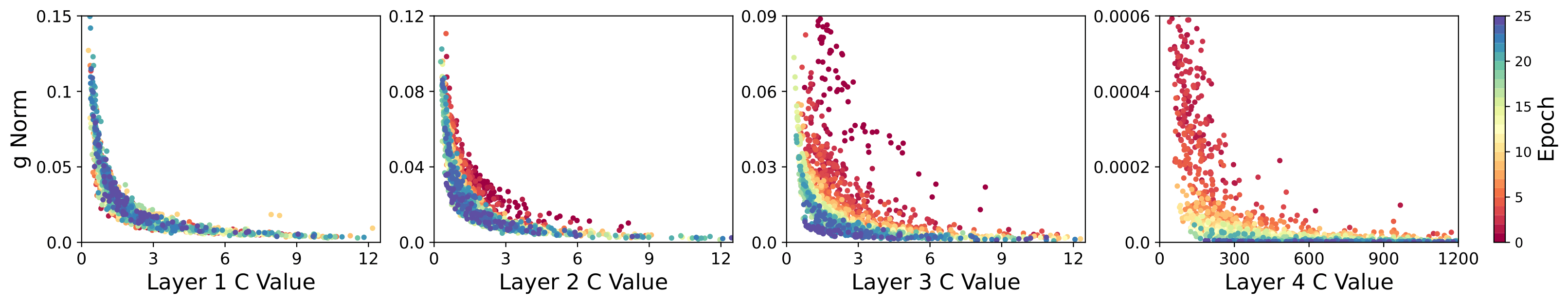}
    \end{adjustbox}
    \caption{
    Throughout the training, we visualize the distribution of c values for different neurons in each layer. On the bottom row we displays the correlation between the norm of the g vector and the c value of the neurons throughout the training process. The g vector behaves similarly to the gradient in error-backpropagation. The plot indicates that when the norm of g increases, the C value decreases and serves as a normalization factor. The network and the training process is identical to the main MNIST experiment. The model consists of four hidden layers, with each layer containing 2000 neurons. ReLU activation is utilized between the layers. }
    \label{fig:DampingEffect}
\end{figure}

\section{Implementation Details}
\label{ImplementationAppendix}
The primary objective of this section is to facilitate a comprehensive understanding of the code structure, thereby simplifying its readability. Additionally, this section provides an in-depth explanation of the technical implementation facilitated through the PyTorch library, which allows for the modification of gradients during the backward pass.

To construct a model leveraging the \algname{} algorithm, we use our customized layers instead of conventional Linear or Convolutional layers. These distinctive layers are designed to mimic the behavior of conventional layers (referred to as \textbf{F}) during the forward pass, yet they differ in their backward propagation: they return $g$ signals instead of gradients. Fig.~\ref{fig:imple} delineates the operation of each of these layers. Initially, we explicate how they generate the output using the layer's weights and input. Subsequently, we elucidate the backward pass, describing the process of deriving weight updates and $g$ signals based on the $g$ signals backpropagated from the succeeding layer. It is paramount to underscore that these layers transmit $g$ signals "As" gradient; therefore, they are understood and processed as gradients within the PyTorch library.

Regarding the forward pass, each layer processes by receiving the weights and inputs. Using these matrices, it computes the output in the usual manner via function F, where F signifies a function that could manifest as a linear or convolutional layer. Nonetheless, the output is not returned directly. Instead, the module invokes its subordinate module, "Fake Gradient Producer," using the output value, weights, and input matrix. The Fake Gradient Producer recalculates the output using the weights and inputs, averages this calculated output with the output provided by the parent module, and returns this value as the layer's output. Although the outputs generated by function F are identical, rendering the second call of F ostensibly redundant or unnecessary from the forward pass's perspective, this reiteration is deliberately integrated to modulate and regulate the backward pass, thus playing an integral part in the overarching process.

To comprehend the backward pass, it is imperative to understand that the Pytorch library allows us to register a hook on modules. This hook is activated immediately before the module returns its gradients, enabling us to modify the gradients. In our context, the Fake Gradient Producer module registers a backward hook, which is triggered when the gradients with respect to the inputs of the Fake Gradient Producer are computed. Within this hook, we can determine the extent to which each neuron/filter/agent updates its weights, which embodies $\frac{1}{c}$ in \algname{}. We subsequently divide the output gradients by this value, similar to how $g$ is calculated using $r$ and $c$. This method guarantees that the $\ell_2$-norm of the weight gradients is $1$, ultimately resulting in an $\alpha$ $\ell_2$-norm of weight update, utilizing an $\alpha$ learning rate.

\begin{figure}
    \centering
    \includegraphics[width=0.70\textwidth]{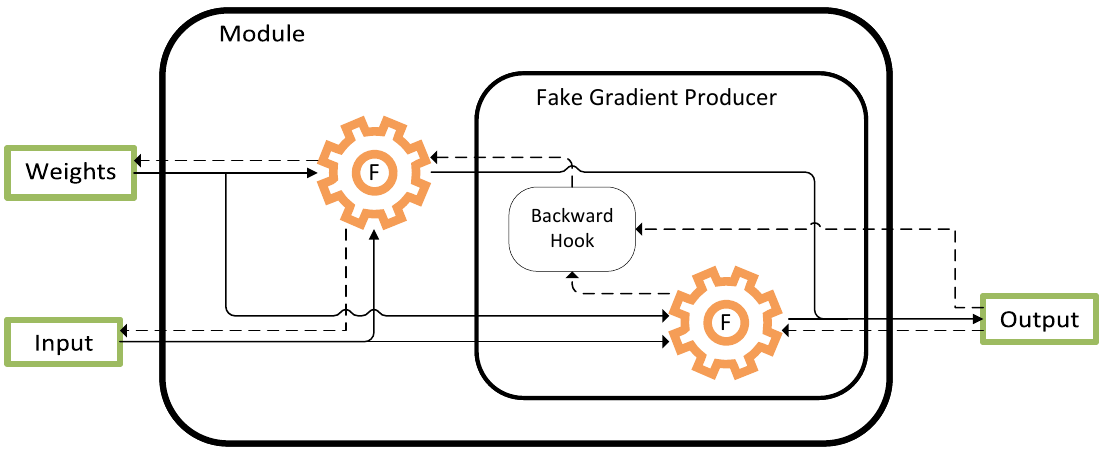}
    \caption{
    The figure presents a schematic representation of the implementation process. Solid lines denote the forward pass, and dashed lines denote the backward pass. Moreover, F represents a function that could be a linear or convolutional module.  
    When two lines converge, their values are averaged.}
    \label{fig:imple}
\end{figure}

Upon looking at Fig.~\ref{fig:imple}, the results presented in Table~\ref{performanceTabel} are intuitive, given that function F is computed twice. This implies that the execution time of our algorithm would be approximately twice the time for the baseline algorithm.

\end{document}